\documentclass{abel}

\usepackage{amssymb}
\usepackage{xspace}
\usepackage{enumitem}
\usepackage{tabularx}
\usepackage{longtable}


\theoremstyle{definition}

\renewcommand{\theassumption}{A\arabic{assumption}}
\crefname{assumption}{assumption}{assumptions}
\Crefname{assumption}{Assumption}{Assumptions}

\tcbset{
  paper theory/.style={
    enhanced,
    breakable,
    colback=abelaqua!15!abelpaperwhite,
    colframe=abeldeepteal,
    boxrule=0.65pt,
    arc=2.2mm,
    left=2.5mm, right=2.5mm, top=2mm, bottom=2mm,
    before skip=7pt, after skip=7pt,
  },
  paper proof/.style={
    enhanced,
    breakable,
    colback=abelsoftcanvas,
    colframe=abelnearblack!30!abelpaperwhite,
    boxrule=0.55pt,
    arc=2.2mm,
    left=2.5mm, right=2.5mm, top=2mm, bottom=2mm,
    before skip=5pt, after skip=7pt,
    before upper={\setlength{\parskip}{\medskipamount}},
  },
  paper callout/.style={
    enhanced,
    breakable,
    colback=abelherowash,
    colframe=abelherowash,
    boxrule=0pt,
    arc=0mm, outer arc=0mm,
    borderline west={2.2pt}{0pt}{abelsignalteal},
    left=4mm, right=3.5mm, top=2.4mm, bottom=2.4mm,
    before skip=7pt, after skip=7pt,
  },
}
\newtcolorbox{paperassumptions}[1][]{paper theory, fonttitle=\bfseries, #1}
\tcolorboxenvironment{theorem}{paper theory}
\tcolorboxenvironment{proposition}{paper theory}
\tcolorboxenvironment{lemma}{paper theory}
\tcolorboxenvironment{corollary}{paper theory}
\tcolorboxenvironment{definition}{paper theory}
\tcolorboxenvironment{example}{paper theory}
\tcolorboxenvironment{proof}{paper proof}

\newcommand{\SNR}{\mathrm{SNR}}
\newcommand{\IC}{\mathrm{IC}}
\newcommand{\MSE}{\mathrm{MSE}}
\newcommand{\Var}{\operatorname{Var}}
\newcommand{\Cov}{\operatorname{Cov}}
\newcommand{\Corr}{\operatorname{corr}}
\newcommand{\E}{\mathbb{E}}
\newcommand{\yhat}{\hat{y}}
\newcommand{\eps}{\varepsilon}
\newcommand{\Ft}{\mathcal{F}_{t-1}}
\newcommand{\ampopt}{\mathcal{A}^{\star}}
\newcommand{\ampraw}{\widetilde{\mathcal{A}}}
\newcommand{\method}{\textsc{CalibRank}\xspace}
\newcommand{\dataset}{Finance1K\xspace}

\newcommand{\tablehead}[1]{{\bfseries\boldmath #1}}

\AtBeginDocument{\DeclareFontShape{T1}{lmr}{bx}{sc}{<->ssub*lmr/bx/n}{}}

\graphicspath{{figures/}}

\renewcommand\authorformat[2][]{\mbox{\small\sffamily\bfseries\color{abelhi} #2$^{#1}$}}
\title{Forecast Collapse in Time-Series Foundation Models}
\author[1,2,*]{Shu Wan}
\author[1,*]{Miles Ma}
\author[1]{Hank Zhu}
\author[1]{Guangqi Liu}
\author[1]{Stephen Wang}
\author[3]{Qingsong Wen}
\author[2]{Huan Liu}
\affiliation[1]{Abel AI Lab}
\affiliation[2]{Arizona State University}
\affiliation[3]{University of Oxford}
\contribution[*]{Equal contribution}
\date{\today}
\metadata[Data]{\url{https://huggingface.co/datasets/abel-lab/finance1k}}
\correspondence{Shu Wan \email{swan@asu.edu}, Stephen Wang \email{stephen@abel.ai}}

\abstract{%
When forecasting hourly returns for 1,000 US equities, we observe an unexpected phenomenon: predictions become nearly flat and show poor stock ranking, as measured by cross-sectional correlation. We call this \emph{forecast collapse}. Surprisingly, the phenomenon largely disappears when forecasting trading volume under the same setting. We investigate forecast collapse across time-series foundation models (TSFMs), twelve deep-learning forecasting models, and 97 public benchmark configurations, and find that it is closely tied to target predictability. We identify two distinct reasons behind it: low predictability limits the amplitude of calibrated point forecasts, while per-series objectives leave cross-series structure unidentified. These findings reveal a calibration-ranking tradeoff: optimizing squared error leads to flat predictions, whereas directly optimizing cross-sectional correlation improves ranking but can inflate forecast amplitude by more than an order of magnitude. To address this tradeoff, we introduce \method, a simple objective that balances calibration and ranking. On \dataset, \method nearly triples cross-sectional correlation while keeping amplitude close to the target, and improves correlation on all tested models. Our results reveal a blind spot in conventional time-series evaluation: per-series metrics can hide failures in cross-series structure needed by downstream decisions.
}

\begin{document}
\maketitle

\section{Introduction}
\label{sec:intro}

\begin{figure}[!t]
  \centering
  \includegraphics[width=\textwidth]{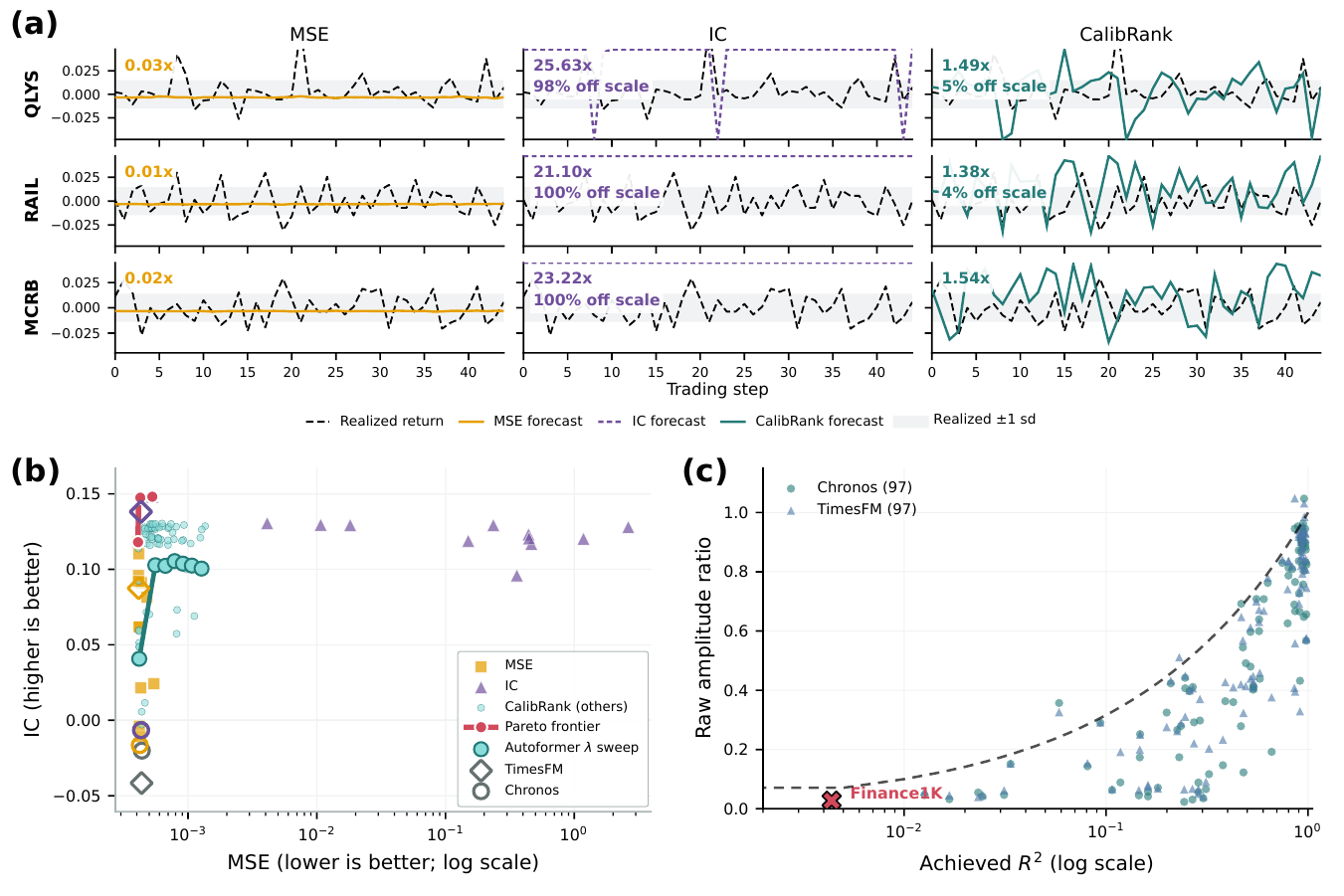}
  \caption{\textbf{Forecast collapse and the calibration-ranking tradeoff.}
    \textbf{(a)} Realized and predicted returns for three equities under MSE,
    IC, and \method; labels report raw amplitude relative to the target.
    \textbf{(b)} MSE against cross-sectional IC across forecasting models and
    objectives; lines mark the empirical Pareto frontier and the Autoformer
    $\lambda$ sweep. \textbf{(c)} Raw forecast amplitude against achieved
    $R^{2}$ for TimesFM and Chronos across 97 GIFT-Eval
    configurations~\citep{aksu2024gift}. \dataset is far less predictable than
    any of them.}
  \label{fig:figure123}
\end{figure}

Time series foundation models (TSFMs)~\citep{liang2024foundation} are large,
pre-trained models designed to forecast across diverse domains with little or
no dataset-specific tuning. Financial markets provide a particularly
challenging stress test. Equity returns have notoriously low signal-to-noise
ratios, yet many financial decisions depend not only on the
accuracy of individual forecasts but also on their relative rank across
many assets, measured by the correlation between predicted and realized
values at each timestamp, commonly known as the information coefficient (IC).

When we apply off-the-shelf TSFMs~\citep{das2024timesfm,ansari2024chronos} to
hourly US equity returns, something unexpected happens: the
predictions are almost flat near zero, and are poorly correlated with the actual stock ranking
(\Cref{fig:figure123}(b)). This finding echoes a recent large-scale
evaluation~\citep{rahimikia2025revisiting} on equity excess returns, which
finds that TSFMs are weak both zero-shot and after fine-tuning. We
refer to this combination of low forecast amplitude and poor cross-sectional
structure as \emph{forecast collapse}.

Surprisingly, forecast collapse largely disappears when we change the target
from equity returns to traded volume on the same stock panel while keeping the
forecasting protocol fixed. The contrast suggests that forecast collapse is
not simply a failure of a particular model or implementation. Instead, it
raises a more fundamental question:

\begin{center}
\emph{Why does the same forecasting pipeline collapse for one target but not
for another?}
\end{center}

We provide a theoretical explanation for forecast collapse through two
distinct mechanisms. First, low target predictability limits the amplitude of
calibrated point forecasts: when little of the target variation is
predictable, calibrated forecasts necessarily shrink relative to the realized
target. Second, per-series objectives leave cross-series structure
unidentified: forecasts with the same per-series risk can have different
cross-sectional correlations. We formalize the first mechanism through an
exact relation between best-scaled amplitude, correlation, and predictable
variance, and the second through an invariance result for per-series risk.

These two results lead to a \emph{calibration-ranking tradeoff}.
A per-series calibration objective, such as mean squared error (MSE), scores
each forecast against its own target but does not directly optimize
cross-sectional ranking. Under low predictability, the resulting calibrated
point forecast can have low amplitude. Rescaling can change
forecast amplitude but cannot recover missing ranking information. In
contrast, directly optimizing IC improves ranking but does not constrain
forecast scale; in our experiments, this produces forecasts over 20 times the
target amplitude (\Cref{fig:figure123}(a)). This tradeoff
suggests that calibration and cross-sectional ranking should be considered
jointly rather than evaluated through either objective alone. Based on this
analysis, we introduce \method, which combines MSE with cross-sectional
correlation, with $\lambda$ controlling their relative importance.

We test these theoretical claims using \dataset, a panel of hourly observations for
1,000 US equities with aligned return and traded-volume targets. Its matched
targets allow us to compare a low-predictability target with a more
predictable target while keeping the forecasting setting fixed. We use
synthetic data with known signal and noise to verify the theoretical
relations directly, and vary $\lambda$ to trace the calibration-ranking
tradeoff predicted by the analysis. We then test whether the same behavior
holds across twelve forecasting architectures and examine the relation
between forecast amplitude and achieved predictability for TimesFM and
Chronos across 97 public GIFT-Eval configurations
(\Cref{fig:figure123}(c)). These experiments serve as tests of the theoretical
claims and of \method as a practical way to balance calibration and
cross-sectional ranking.

More broadly, forecast collapse exposes a limitation of conventional
time-series evaluation. When a multivariate panel is evaluated as independent
series, per-series accuracy can remain informative while the cross-series
structure used by downstream decisions is never measured. Forecasting
benchmarks should therefore preserve and evaluate this structure when it
matters to the downstream task.

Our contributions are:
\begin{itemize}[leftmargin=1.2em,itemsep=1pt,topsep=2pt]
  \item We identify \emph{forecast collapse}, low forecast amplitude together
    with poor cross-sectional structure, and show a sharp target contrast
    between low-predictability stock returns and higher-predictability traded
    volume under a matched forecasting protocol.
  \item We provide a theoretical explanation for forecast collapse by
  characterizing two distinct mechanisms. Predictable variance limits the
  amplitude of calibrated point forecasts, while per-series risk leaves
  cross-series coupling unidentified. We derive the exact relation between
  best-scaled amplitude, forecast-target correlation, and predictable
  variance.
  \item We characterize the resulting calibration-ranking tradeoff and
  introduce \method, a simple objective that balances calibration and
  cross-sectional ranking. We validate the predicted tradeoff across twelve
  forecasting models and show that \method improves cross-sectional
  correlation for every model while maintaining control of forecast scale.
  \item We release \dataset, an hourly panel of 1,000 US equities with aligned
return and volume targets, which supports controlled comparisons of
target predictability and training objectives under the same forecasting
protocol.
\end{itemize}

\FloatBarrier
\section{Related Work}
\label{sec:related}

\paragraph{Ceilings and collapse.}
\citet{green2026expectations} studies how mean-optimal forecasting
under-disperses the marginal law, which is what a raw amplitude below one
records. \citet{andreoletti2026collapse} independently documents
low-amplitude transformer forecasts on returns. Our analysis relates
best-scaled amplitude to a declared predictability ceiling and separately
measures cross-sectional ordering. \citet{zhang2026nontrivial} bounds
out-of-sample $R^{2}$ through hit rates, while ours uses
predictable variance. Related variance-share limits appear in psychometrics and
ecological forecasting \citep{shrout1979intraclass,pennekamp2019intrinsic}.
Classical measurement-error theory supplies the attenuation identity used in
our derivation \citep{fuller1987measurement,carroll2006measurement}. Here it
connects a familiar statistical limit to best-scaled forecast amplitude and
separates that quantity from raw model scale.

\paragraph{Related ideas in other domains.}
Image restoration has shown that distortion and perceptual quality can trade
off along a frontier \citep{blau2018perception,freirich2021distortion}, motivating
structural criteria beyond pixelwise $L_{2}$
\citep{wang2004ssim,wangbovik2009mse}. This is a useful precedent for evaluating
more than pointwise error, but the quantities differ: its perception index is a
divergence between marginal laws, whereas IC measures dependence between
forecasts and targets. Deterministic weather models provide a second parallel:
training toward a conditional mean can produce smooth, low-amplitude fields,
which has motivated generative ensembles that preserve variability
\citep{price2025gencast}.

\paragraph{Regression and ranking objectives.}
Composite regression-ranking losses are established in stock prediction and
information retrieval. \citet{feng2019temporal} combines squared error with a
pairwise hinge, \citet{lin2026lambdarankic} optimizes rank correlation, and
\citet{yang2020qlib} provides correlation losses in a forecasting toolkit.
Differentiable ranking operators are also well developed
\citep{blondel2020fast,swezey2021pirank}. \citet{bai2023rcr} shows that carefully
designed listwise objectives can reduce conflicts between regression and
ranking. These results motivate stronger composites that may improve on the
simple scalarization evaluated here.

Section~\ref{sec:objectives} places this comparison in the current
TSFM landscape. Fifteen surveyed models disclose a forecasting loss, four use a
squared-error-like loss, and the remainder use quantile, likelihood,
cross-entropy, or flow-matching objectives. None of the eighteen sourced
training objectives scores a whole cross-section at once. The issue
is therefore broader than MSE: point forecasts have different calibration
properties, but per-series scoring leaves the same multivariate axis absent.

\paragraph{Joint and foundation-model forecasting.}
Probabilistic forecasters such as GPVar and TACTiS model richer dependence
directly \citep{salinas2019gpvar,drouin2022tactis}. \method addresses the narrower
cross-sectional functional the downstream decision uses. Current TSFMs
span autoregressive, masked-reconstruction, and patch-based objectives
\citep{das2024timesfm,ansari2024chronos,woo2024moirai,goswami2024moment,
ekambaram2024ttm,rasul2023lagllama}, yet the sourced objectives in our survey do
not score the cross-section jointly. The competitiveness of simple forecasting
baselines \citep{zeng2023transformers} further supports evaluating objectives
and protocols independently of architectural scale.

\section{Preliminaries}
\label{sec:preliminaries}

This section fixes the objects that are easy to conflate in panel forecasting:
the information available before a prediction, the predictable component of
the target, the point forecast a model outputs, and the cross-sectional
functional evaluated downstream. It then places these objects against the
training objectives used by current time-series foundation models.

\subsection{Forecasting panel and evaluation unit}
\label{sec:notation}

At forecast timestamp $t$, the $N$ series form one cross-section. A model uses the
history before $t$ to produce one point forecast $\yhat_{it}$ for each realized
target $y_{it}$. We collect the targets and forecasts into vectors and define
temporal reporting from timestamp-level statistics:
\begin{equation}
\begin{aligned}
  \mathbf{y}_{t}
  &:= (y_{1t},\ldots,y_{Nt})^{\top},
  & \widehat{\mathbf{y}}_{t}
  &:= (\yhat_{1t},\ldots,\yhat_{Nt})^{\top}, \\
  g_t
  &:= g(\widehat{\mathbf{y}}_{t},\mathbf{y}_{t}),
  & \overline{g}
  &:= \frac{1}{T}\sum_{t=1}^{T} g_t .
\end{aligned}
\label{eq:panel-evaluation}
\end{equation}
Thus the cross-section at one timestamp is the evaluation unit, and timestamps
are the replication unit. In this paper, $g_t$ includes cross-sectional
correlation, raw amplitude, and per-timestamp squared error. For a complete
panel, averaging per-timestamp squared error is equivalent to averaging all
scalar errors. Correlation is different: pooling all $(i,t)$ pairs would mix
temporal and cross-sectional variation and would no longer score the ordering
available to a decision made at timestamp $t$.

The notation also separates three layers of the argument. The decomposition
$y_{it}=\mu_{it}+\eps_{it}$ describes a target-side predictability limit under
the declared information set. The model error
$\delta_{it}=\yhat_{it}-\mu_{it}$ describes alignment with that predictable
component. Finally, $\IC_t$, raw amplitude, and best-scaled amplitude
describe different properties of the forecast itself. A forecast can therefore
have weak scale, weak ordering, or both; the paper does not treat these as the
same failure. \Cref{tab:notation} groups the recurring symbols by these roles.
Subscripts are suppressed when a statement concerns a generic forecast-target
pair.

\begin{table}[!htbp]
  \centering
  \footnotesize
  \setlength{\tabcolsep}{4pt}
  \renewcommand{\arraystretch}{1.14}
  \begin{tabularx}{\textwidth}{@{}
      >{\raggedright\arraybackslash}p{0.16\textwidth}
      >{\raggedright\arraybackslash}X
      @{\hspace{10pt}}
      >{\raggedright\arraybackslash}p{0.16\textwidth}
      >{\raggedright\arraybackslash}X@{}}
    \toprule
    \multicolumn{2}{@{}c}{\tablehead{Panel, signal, and evaluation}} &
    \multicolumn{2}{c@{}}{\tablehead{Forecast behavior and optimization}} \\
    \cmidrule(r){1-2}\cmidrule(l){3-4}
    \tablehead{Symbol} & \tablehead{Definition and role} &
    \tablehead{Symbol} & \tablehead{Definition and role} \\
    \midrule
    $i,t,N,T,L$
      & Series/time indices; numbers of series and evaluated timestamps; context
        length.
      & $\rho$
      & Alignment $\Corr(\yhat,\mu)$ with the predictable component. \\
    $\mathbf{y}_{t},\widehat{\mathbf{y}}_{t}$
      & Realized and forecast cross-sections at timestamp $t$.
      & $\IC_t,\overline{\IC}$
      & Within-timestamp Pearson correlation and its signed temporal mean. \\
    $\Ft$
      & Information available to every admissible forecast before $t$.
      & $\ampraw_t,\ampopt$
      & Raw output scale and best-scaled amplitude; only the latter obeys the
        amplitude-correlation identity. \\
    $g_t,\overline{g}$
      & Timestamp-level statistic and temporal average in
        \eqref{eq:panel-evaluation}.
      & $s^{\star}$
      & MSE-optimal slope for a centered forecast. \\
    $\mu_{it},\eps_{it},\sigma_i^{2}$
      & Conditional mean, innovation, and irreducible variance
        $\E[\eps_{it}^{2}]$.
      & $\mathcal{L},\lambda$
      & Composite objective and calibration-ranking tradeoff. \\
    $\delta_{it},Q_i$
      & Error $\yhat_{it}-\mu_{it}$ and reducible error
        $Q_i=\E[\delta_{it}^{2}]$.
      & $B,M$
      & Batch timestamps and scalar errors; $M=N\lvert B\rvert$ for a complete
        panel. \\
    $\SNR,R^{2}_{\max},\IC_{\max}$
      & Signal-to-noise ratio, forecastable variance share, and correlation
        ceiling under $\Ft$.
      & $c_t$
      & Predictability proxy used only by the weighted ablation. \\
    \bottomrule
  \end{tabularx}
  \caption{Notation. Quantities indexed by $t$ are computed within a
    cross-section before temporal averaging. Raw amplitude $\ampraw_t$ is
    distinct from best-scaled amplitude $\ampopt$.}
  \label{tab:notation}
\end{table}

\subsection{Training objectives of time-series foundation models}
\label{sec:objectives}

\Cref{tab:objectives} audits the objective stated in each source paper. The
survey covers nineteen models and eighteen sourced objectives; TimeGPT names no
loss function and is omitted from the table. Fifteen models disclose a
forecasting loss, of which four are squared-error-like. The entries excluded
from that denominator train masked reconstruction or fit no loss at all. Three
objectives carry more than one term: Time-MoE adds mixture-of-experts load
balancing, TSPulse a cross-entropy term on the log-magnitude spectrum, and Toto
a Cauchy robust point term. The table marks both cases.

The table asks two questions of each objective. The per-series columns record
which conditional quantity of a single series it identifies. A scoring function
is consistent for a functional when that functional minimizes its expected
score \citep{gneiting2011making}. Squared error and
its robust surrogates pin the conditional mean. A pinball grid pins the levels
it contains, so it delivers the median exactly when $0.5$ lies on the grid,
which it does for all six quantile-trained models. Likelihood, cross-entropy,
flow-matching, and in-context prior fitting fit a predictive distribution, off
which the mean, the median, and arbitrary quantiles can all be read. Likelihood
and cross-entropy are strictly proper scoring rules
\citep{gneiting2007scoring}; flow matching regresses a velocity field rather
than a value, so the same quantities come from its draws. The last
column asks whether the objective scores a whole cross-section at once, and it
answers the same way for every row: in the sourced formulations, the training
score breaks into a sum over individual series or forecast tokens, with no term
that depends on two series together, so nothing in it scores the cross-section
of \eqref{eq:panel-evaluation}. The families therefore differ
widely in what they identify about a single series and agree completely in what
they leave out.

\definecolor{objMSE}{HTML}{B4503A}
\definecolor{objQuantile}{HTML}{2F7D95}
\definecolor{objNLL}{HTML}{63558F}
\definecolor{objXent}{HTML}{9C7622}
\definecolor{objFlow}{HTML}{44795A}
\definecolor{objPrior}{HTML}{7E6350}
\begin{table}[!htbp]
  \centering
  \small
  \caption{\textbf{Training objectives of surveyed TSFMs.} No sourced
    objective scores the cross-section jointly.
    \textsuperscript{\dag}~Not counted among the fifteen disclosed
    forecasting losses. \textsuperscript{\ddag}~Multi-term objective.}
  \label{tab:objectives}
  \setlength{\tabcolsep}{4pt}
  \renewcommand{\arraystretch}{1.15}
  \begin{tabular}{@{}ll cccc c@{}}
    \toprule
    & & \multicolumn{4}{c}{\tablehead{Identified per series}} & \multicolumn{1}{c}{\tablehead{Scored jointly}} \\
    \cmidrule(lr){3-6}\cmidrule(l){7-7}
    \tablehead{Model} & \tablehead{Loss} & \tablehead{Mean} & \tablehead{Median} & \tablehead{Quantiles} & \tablehead{Dist.} & \tablehead{Cross-section} \\
    \midrule
    TimesFM~\citep{das2024timesfm} & {\color{objMSE}Squared error} & {\color{objMSE}$\checkmark$} &  &  &  & $\times$ \\
    TTM~\citep{ekambaram2024ttm} & {\color{objMSE}Squared error} & {\color{objMSE}$\checkmark$} &  &  &  & $\times$ \\
    Timer~\citep{liu2024timer} & {\color{objMSE}Squared error} & {\color{objMSE}$\checkmark$} &  &  &  & $\times$ \\
    Time-MoE~\citep{shi2024timemoe} & {\color{objMSE}Squared error}\textsuperscript{\ddag} & {\color{objMSE}$\checkmark$} &  &  &  & $\times$ \\
    MOMENT\textsuperscript{\dag}~\citep{goswami2024moment} & {\color{objMSE}Squared error} & {\color{objMSE}$\checkmark$} &  &  &  & $\times$ \\
    TSPulse\textsuperscript{\dag}~\citep{ekambaram2025tspulse} & {\color{objMSE}Squared error}\textsuperscript{\ddag} & {\color{objMSE}$\checkmark$} &  &  &  & $\times$ \\
    \addlinespace[2pt]
    Chronos-Bolt~\citep{ansari2024chronosbolt} & {\color{objQuantile}Quantile} &  & {\color{objQuantile}$\checkmark$} & {\color{objQuantile}$\checkmark$} &  & $\times$ \\
    Chronos-2~\citep{ansari2025chronos2} & {\color{objQuantile}Quantile} &  & {\color{objQuantile}$\checkmark$} & {\color{objQuantile}$\checkmark$} &  & $\times$ \\
    TiRex~\citep{auer2025tirex} & {\color{objQuantile}Quantile} &  & {\color{objQuantile}$\checkmark$} & {\color{objQuantile}$\checkmark$} &  & $\times$ \\
    Kairos~\citep{feng2025kairos} & {\color{objQuantile}Quantile} &  & {\color{objQuantile}$\checkmark$} & {\color{objQuantile}$\checkmark$} &  & $\times$ \\
    YingLong~\citep{wang2025yinglong} & {\color{objQuantile}Quantile} &  & {\color{objQuantile}$\checkmark$} & {\color{objQuantile}$\checkmark$} &  & $\times$ \\
    FlowState~\citep{graf2025flowstate} & {\color{objQuantile}Quantile} &  & {\color{objQuantile}$\checkmark$} & {\color{objQuantile}$\checkmark$} &  & $\times$ \\
    \addlinespace[2pt]
    Moirai~\citep{woo2024moirai} & {\color{objNLL}Likelihood} & {\color{objNLL}$\checkmark$} & {\color{objNLL}$\checkmark$} & {\color{objNLL}$\checkmark$} & {\color{objNLL}$\checkmark$} & $\times$ \\
    Lag-Llama~\citep{rasul2023lagllama} & {\color{objNLL}Likelihood} & {\color{objNLL}$\checkmark$} & {\color{objNLL}$\checkmark$} & {\color{objNLL}$\checkmark$} & {\color{objNLL}$\checkmark$} & $\times$ \\
    Toto~\citep{cohen2025toto} & {\color{objNLL}Likelihood}\textsuperscript{\ddag} & {\color{objNLL}$\checkmark$} & {\color{objNLL}$\checkmark$} & {\color{objNLL}$\checkmark$} & {\color{objNLL}$\checkmark$} & $\times$ \\
    \addlinespace[2pt]
    Chronos~\citep{ansari2024chronos} & {\color{objXent}Cross-entropy} & {\color{objXent}$\checkmark$} & {\color{objXent}$\checkmark$} & {\color{objXent}$\checkmark$} & {\color{objXent}$\checkmark$} & $\times$ \\
    \addlinespace[2pt]
    Sundial~\citep{liu2025sundial} & {\color{objFlow}Flow matching} & {\color{objFlow}$\checkmark$} & {\color{objFlow}$\checkmark$} & {\color{objFlow}$\checkmark$} & {\color{objFlow}$\checkmark$} & $\times$ \\
    \addlinespace[2pt]
    TabPFN-TS\textsuperscript{\dag}~\citep{hoo2025tabpfnts} & {\color{objPrior}Prior fit} & {\color{objPrior}$\checkmark$} & {\color{objPrior}$\checkmark$} & {\color{objPrior}$\checkmark$} & {\color{objPrior}$\checkmark$} & $\times$ \\
    \bottomrule
  \end{tabular}
\end{table}

The survey is about scope, not performance. Forecast quality is itself
multi-axis, maximizing sharpness subject to calibration
\citep{gneiting2007calibration}, and a distributional objective can represent
conditional uncertainty even when a single point forecast has low amplitude. Once that forecast is used by a
cross-sectional decision, however, its output scale and within-timestamp
ordering remain separate quantities that must be measured. \Cref{sec:theory}
formalizes the resulting predictability ceiling and coupling gap, and
\Cref{sec:method} introduces the objective used to score the missing functional.

\section{Why Forecasts Collapse}
\label{sec:theory}

Call an objective \emph{pointwise} when its population risk is a sum of
per-series, per-timestamp terms,
$R(\yhat)=\sum_{i,t}\E[\ell(\yhat_{it},y_{it})]$. Squared error, absolute error,
quantile, likelihood, and cross-entropy losses commonly have this form. Low
predictability constrains the amplitude of a calibrated point forecast.
Pointwise scoring creates a separate identification gap for cross-series
dependence.

\subsection{Predictability and best-scaled amplitude}

All quantities in this subsection use the same sampling distribution. For the
Finance1K analysis, the distribution is one timestamp's cross-section; empirical
metrics are averaged only after the per-timestamp quantities are computed. We
suppress the timestamp index below. Let
\begin{equation}
  y=\mu+\eps,\qquad \mu=\E[y\mid\Ft],\qquad \E[\eps\mid\Ft]=0 ,
  \label{eq:decomp}
\end{equation}
where $\Ft$ is the information available to every competing forecast. The law
of total variance gives, with
$\SNR:=\Var(\mu)/\E[\Var(y\mid\Ft)]$,
\begin{equation}
  R^{2}_{\max}:=\frac{\Var(\mu)}{\Var(y)}
  =\frac{\SNR}{1+\SNR},
  \qquad
  \IC_{\max}:=\sqrt{R^{2}_{\max}} .
  \label{eq:volpred}
\end{equation}
This is the forecastable share of variance under the declared information set.
For an $\Ft$-measurable forecast $\yhat$, write
$\rho=\Corr(\yhat,\mu)$. Orthogonality of $\eps$ to $\Ft$ yields
\begin{equation}
  \Corr(\yhat,y)=\rho\,\IC_{\max}.
  \label{eq:attenuation}
\end{equation}
This is the standard attenuation identity
\citep{fuller1987measurement,carroll2006measurement}.

For a centered forecast-target pair, the scalar minimizing
$\E[(s\yhat-y)^{2}]$ is
$s^{\star}=\Cov(\yhat,y)/\Var(\yhat)$.

\begin{proposition}[Amplitude after optimal rescaling]
\label{prop:collapse}
For any centered $(\yhat,y)$ with $\Var(\yhat)>0$,
\begin{equation}
  \ampopt
  :=\frac{\mathrm{sd}(s^{\star}\yhat)}{\mathrm{sd}(y)}
  =\lvert\Corr(\yhat,y)\rvert.
  \label{eq:collapse}
\end{equation}
If, in addition, \eqref{eq:decomp} holds and $\yhat$ is
$\Ft$-measurable, then $\ampopt\le\IC_{\max}$.
\end{proposition}

\begin{proof}[Proof sketch]
Substituting $s^{\star}$ gives
$\mathrm{sd}(s^{\star}\yhat)/\mathrm{sd}(y)
=\lvert\Cov(\yhat,y)\rvert/
(\mathrm{sd}(\yhat)\mathrm{sd}(y))$. Under the additional conditions, the
bound follows from \eqref{eq:attenuation} and $\lvert\rho\rvert\le1$.
\end{proof}

The equality is the ordinary-least-squares slope identity in
standard-deviation units. It applies per timestamp as
$\ampopt_{t}=\lvert\IC_{t}\rvert$. The experiments report two different
descriptive quantities: signed mean $\IC_t$ and the raw amplitude
\begin{equation}
  \ampraw_{t}:=\frac{\mathrm{sd}_{i}(\yhat_{it})}
                         {\mathrm{sd}_{i}(y_{it})}.
  \label{eq:raw-amplitude}
\end{equation}
No identity equates $\ampraw_t$ with $\IC_t$. Squared-error training approaches
the optimal scaling only to the extent that the fitted model approaches its
population optimum.

The distinction also clarifies what the proposition predicts. It constrains the
amplitude after the best scalar recalibration of a given forecast. A model may
output any raw scale before recalibration, including the large values produced by
a pure correlation loss. Comparing raw amplitude with correlation can diagnose
residual scale error, but a ratio of their time averages is not an estimate of
$s_t^\star$: the slope is timestamp-specific and involves a ratio of covariance to
variance. For pretrained models with different objectives, raw amplitude is an
empirical diagnostic rather than the left-hand side of \eqref{eq:collapse}.

Section~\ref{sec:results} tests these empirical implications on Finance1K,
synthetic panels, and public foundation-model benchmarks.

The amplitude argument also extends beyond squared-error training. Of the
nineteen foundation models surveyed in Section~\ref{sec:objectives}, fifteen
disclose a forecasting objective and only four are squared-error-like.

\begin{proposition}[Mean and median point forecasts]
\label{prop:quantile}
Suppose $y\mid\Ft=\mu+\sigma Z$ for a fixed distribution $Z$ symmetric about
zero with finite mean, and a scale $\sigma$ common across series. The
conditional mean and median equal $\mu$, so their optimally scaled forecasts
satisfy \eqref{eq:collapse}.
\end{proposition}

\begin{proof}[Proof sketch]
Symmetry about zero and a finite mean give $\E[Z]=0$, and zero is a median of
$Z$. Both forecasts are therefore $\mu$, and \Cref{prop:collapse} applies to the
pair $(\mu,y)$. \Cref{sec:formal} gives the details.
\end{proof}

The result concerns a point forecast. A well-specified predictive distribution
can retain the target's conditional spread even when its mean or median is
small. Series-varying scale and off-median quantiles require separate treatment,
given in Section~\ref{sec:formal}.

\paragraph{The two conditions are distinct.}
Low predictability is the condition for amplitude attenuation, rather than a
complete explanation of forecast collapse. A low-signal univariate target can
produce a small but well-calibrated point forecast without any
cross-sectional question. Conversely, a highly predictable multivariate target
can have well-scaled marginal forecasts while its joint ordering is wrong. The
pattern in Figure~1(a) combines the two: weak predictable variance
limits the calibrated amplitude, and per-series scoring leaves the
cross-series relationship underdetermined.

\subsection{Per-series risk leaves coupling unidentified}

Amplitude is one part of collapse. Cross-sectional ordering depends on the
joint behavior of all series at one timestamp.

\begin{proposition}[Insensitivity to cross-series coupling]
\label{prop:separable}
Let $R$ be pointwise. Then $R$ is a functional of the per-coordinate joint laws
$\{\mathrm{Law}(\yhat_{it},y_{it})\}$ alone. It is unchanged when those pairs
are recoupled across $i$ without changing any per-coordinate law, whereas
$\IC_t$ can change. Hence equal-risk forecasters can have different
cross-sectional correlation.
\end{proposition}

\begin{proof}[Proof sketch]
Every summand of $R(\yhat)=\sum_{i,t}\E[\ell(\yhat_{it},y_{it})]$ is an
expectation of a function of the single pair $(\yhat_{it},y_{it})$, so $R$ only reads
the per-coordinate laws. Replacing the panel's joint law by
another one under which every pair keeps its law therefore leaves $R$ fixed.
Cross-sectional correlation is computed from the two vectors
$(\yhat_{it})_{i}$ and $(y_{it})_{i}$ at one timestamp, so it depends on
exactly the coupling that such a replacement is free to move.
\Cref{ex:coupling} carries out the replacement explicitly.
\end{proof}

A two-series Gaussian construction holds squared-error risk at $1.40$ while
changing $\E[\IC_t]$ from $0.371$ to $0.032$; the explicit covariance
construction and admissibility conditions appear in Section~\ref{sec:formal}.
This is an invariance result.
With sufficient capacity, the population squared-error minimizer
$\yhat=\mu$ can order the cross-section well. The proposition says that
per-series risk alone does not identify the coupling, which motivates measuring
and rewarding the cross-sectional functional used downstream.

The language of coupling is useful because marginal accuracy does not determine
a multivariate forecast. By Sklar's theorem \citep{sklar1959fonctions}, a joint
law separates into marginal distributions and a copula carrying dependence.
Pointwise evaluation probes the per-series forecast-target laws; it does not
fully determine the dependence among forecast-target pairs across series.
\method does not estimate a copula. It rewards one observable functional of that
dependence, the per-timestamp cross-sectional correlation used by the downstream
decision. Models that require calibrated scenarios or tail dependence need a
richer joint objective.

\paragraph{Implication for optimization.}
The population MSE solution $\E[y\mid\Ft]$ can rank the cross-section as well as
the information set allows. The invariance result applies when optimization
selects among finite-model solutions with similar marginal risk. Optimization,
regularization, pretraining, and shared representations choose among those
solutions without a guarantee on the downstream functional. A cross-sectional loss
supplies that preference explicitly. The empirical comparisons show that it
improves ordering under the matched protocol studied here.

\FloatBarrier
\section{Formal Foundations and Proofs}
\label{sec:formal}

This section carries the assumptions, the full proofs of the results
stated in \Cref{sec:theory}, and the explicit construction behind
\Cref{prop:separable}. The classical direction law closes the section as an
illustrative consequence. Only the assumptions used throughout are collected
here; one needed by a single result is stated in the subsection that proves it,
so each result carries its own hypotheses where they are used.

\subsection{Assumptions}

\begin{paperassumptions}
\begin{itemize}[leftmargin=1.6em,itemsep=1pt,topsep=2pt]
  \refstepcounter{assumption}\item[\textbf{\theassumption}]\label{a:additive}%
    \emph{Conditional-mean decomposition:}
    $\mu=\E[y\mid\Ft]$ and $\eps=y-\mu$, so $\E[\eps\mid\Ft]=0$.
  \refstepcounter{assumption}\item[\textbf{\theassumption}]\label{a:noleak}%
    \emph{Admissible forecast:} $\yhat$ is $\Ft$-measurable. Therefore
    $\delta:=\yhat-\mu$ is also $\Ft$-measurable.
  \refstepcounter{assumption}\item[\textbf{\theassumption}]\label{a:weak}%
    \emph{Weak signal:} $\SNR\ll1$, which holds for hourly returns and not for
    volume.
\end{itemize}
\end{paperassumptions}

Every claim below is also checked on a synthetic panel that satisfies
\eqref{eq:decomp} directly by drawing $\mu$, $\eps$, and $\delta$ with prescribed
variances, so the quantities are exact ground truth rather than estimates and no
market data enters.

\subsection{Additivity of squared error}

\begin{proposition}[Additive decomposition]
\label{prop:mse}
Under \ref{a:additive} and \ref{a:noleak}, $\MSE=\E[(\yhat-y)^{2}]=Q+\sigma^{2}$
with $Q=\E[\delta^{2}]$ and $\sigma^{2}=\E[\eps^{2}]$.
\end{proposition}

\begin{proof}
\emph{Step 1, the residual carries no signal.} By \ref{a:additive} the target
splits as $y=\mu+\eps$, and measuring the forecast against the same conditional
mean, $\delta:=\yhat-\mu$, writes it as $\yhat=\mu+\delta$. Subtracting, $\mu$
appears in both and cancels:
\begin{equation}
  \yhat-y=(\mu+\delta)-(\mu+\eps)=\delta-\eps .
  \label{eq:resid}
\end{equation}
This is the point of the decomposition. Whatever the forecast knows about the
conditional mean has already been spent; what is left is the model's own error
$\delta$ set against the unpredictable part $\eps$.

\emph{Step 2, square and expand.} Squaring \eqref{eq:resid} and taking
expectations gives three terms:
\begin{equation}
\begin{aligned}
  \MSE
  &=\E\big[(\delta-\eps)^{2}\big] \\
  &=\E[\delta^{2}]-2\,\E[\delta\eps]+\E[\eps^{2}] \\
  &=\E[\delta^{2}]-2\,\Cov(\delta,\eps)+\E[\eps^{2}] .
\end{aligned}
\label{eq:mse-decomposition}
\end{equation}
The last line replaces the raw cross moment by a covariance, which is legitimate
because $\E[\eps]=\E\big[\E[\eps\mid\Ft]\big]=0$ by the tower property, so the
two agree.

\emph{Step 3, the cross term vanishes.} Here \ref{a:noleak} does the work.
Because $\delta$ is $\Ft$-measurable it passes out of the inner conditional
expectation as a constant:
\begin{equation}
\begin{aligned}
  \E[\delta\eps]
  =\E\!\left[\delta\,\E[\eps\mid\Ft]\right]=0.
\end{aligned}
\label{eq:orthogonality-cross-term}
\end{equation}
A forecast that used information outside $\Ft$ would break exactly here:
$\delta$ could then correlate with $\eps$, and $\MSE$ would pick up a cross term
of either sign.

\emph{Step 4, collect.} With the cross term gone,
$\MSE=\E[\delta^{2}]+\E[\eps^{2}]=Q+\sigma^{2}$. Both parts are nonnegative and
only $Q$ responds to the forecast, so $\sigma^{2}$ is a floor no admissible
forecast can get under. The decomposition is an exact identity, not an
approximation, under the declared information set.
\end{proof}

This is the classical reducible and irreducible split of the bias-variance
decomposition \citep{geman1992biasvariance,hastie2009elements}. Root mean
squared error is a monotone transform and ranks identically.

The next result needs one further condition on the cross-section, used nowhere
else.

\begin{paperassumptions}[title={Assumption for \protect\Cref*{cor:mseranksvol}}]
\begin{itemize}[leftmargin=1.6em,itemsep=1pt,topsep=2pt]
  \refstepcounter{assumption}\item[\textbf{\theassumption}]\label{a:hetero}%
    \emph{Heteroscedastic dispersion:} $\Var_{i}(\sigma_{i}^{2})\gg
    \Var_{i}(Q_{i})$ \citep{bollerslev1986generalized}.
\end{itemize}
\end{paperassumptions}

\begin{corollary}[Squared-error rankings can be dominated by irreducible noise]
\label{cor:mseranksvol}
Under \ref{a:hetero}, if $\Var_{i}(Q_{i})\ll\Var_{i}(\sigma_{i}^{2})$ across
series, then $\Corr_{i}(\MSE_{i},\sigma_{i}^{2})\to1$.
\end{corollary}

\begin{proof}
\emph{Step 1.} By \Cref{prop:mse} each series satisfies
$\MSE_{i}=Q_{i}+\sigma_{i}^{2}$, so the quantity in question is the correlation
of a sum with one of its own two parts. Bilinearity splits the numerator
accordingly:
\begin{equation}
  \Cov_{i}(\MSE_{i},\sigma_{i}^{2})
  =\Cov_{i}(Q_{i},\sigma_{i}^{2})+\Var_{i}(\sigma_{i}^{2}).
  \label{eq:mserank-cov}
\end{equation}
A series is therefore ranked by squared error partly on its model error and
partly on its noise level, and the second term is present no matter how the
model behaves.

\emph{Step 2.} Write
\begin{equation}
  v:=\frac{\Var_{i}(Q_{i})}{\Var_{i}(\sigma_{i}^{2})},
  \label{eq:mserank-v}
\end{equation}
which is well defined because \ref{a:hetero} makes the denominator positive, and
small precisely when \ref{a:hetero} holds. It measures how much the reducible
part varies across series relative to the irreducible part.

\emph{Step 3.} Cauchy-Schwarz applied to the two centered
cross-sectional variables gives
$\lvert\Cov_{i}(Q_{i},\sigma_{i}^{2})\rvert\le
\mathrm{sd}_{i}(Q_{i})\,\mathrm{sd}_{i}(\sigma_{i}^{2})
=\sqrt{v}\,\Var_{i}(\sigma_{i}^{2})$. The covariance is thus $O(\sqrt{v})$ once
measured in units of $\Var_{i}(\sigma_{i}^{2})$, and the same bound controls the
denominator through
$\Var_{i}(\MSE_{i})=\Var_{i}(\sigma_{i}^{2})
\big(1+2\Cov_{i}(Q_{i},\sigma_{i}^{2})/\Var_{i}(\sigma_{i}^{2})+v\big)$.

\emph{Step 4.} Dividing \eqref{eq:mserank-cov} by the two standard
deviations and factoring $\Var_{i}(\sigma_{i}^{2})$ out of both puts every term
on one scale:
\begin{equation}
  \Corr_{i}(\MSE_{i},\sigma_{i}^{2})
  =\frac{\Cov_{i}(Q_{i},\sigma_{i}^{2})+\Var_{i}(\sigma_{i}^{2})}
        {\sqrt{\Var_{i}(Q_{i}+\sigma_{i}^{2})\,\Var_{i}(\sigma_{i}^{2})}}
  =\frac{1+O(\sqrt{v})}{\sqrt{1+O(\sqrt{v})+v}} ,
  \label{eq:mserank-ratio}
\end{equation}
which tends to $1$ as $v\to0$.

\emph{Step 5.} The approach is $O(\sqrt{v})$, not $O(v)$. A small
dispersion in $Q_{i}$ still reaches the numerator through its covariance with
$\sigma_{i}^{2}$, which carries one power of $\mathrm{sd}_{i}(Q_{i})$, whereas
its own variance carries two. The gap therefore closes slowly: driving $v$ down
by a factor of four only halves it, so a squared-error ranking stays close to a
noise ranking well after the model errors have been made comparable.
\end{proof}

Thus heterogeneous noise can dominate a cross-sectional squared-error ranking
of series. This diagnostic follows the tradition of metric-pathology analyses
\citep{hyndman2006measures}.

\subsection{Attenuation and the amplitude identity}

\begin{proposition}[Attenuation]
\label{prop:ic}
Under \ref{a:additive} and \ref{a:noleak},
$\IC=\rho\sqrt{\SNR/(1+\SNR)}$, where
$\SNR=\Var(\mu)/\E[\Var(y\mid\Ft)]$.
\end{proposition}

\begin{proof}
\emph{Step 1.} Split $y=\mu+\eps$ and use
bilinearity. The second piece vanishes: $\yhat$ is $\Ft$-measurable by
\ref{a:noleak}, so $\E[\yhat\eps]=\E\big[\yhat\,\E[\eps\mid\Ft]\big]=0$ by
\ref{a:additive}, and $\E[\eps]=0$ by the same assumption, so the covariance is
the product moment. Hence
\begin{equation}
\begin{aligned}
  \Cov(\yhat,y)
  &=\Cov(\yhat,\mu)+\Cov(\yhat,\eps) \\
  &=\Cov(\yhat,\mu).
\end{aligned}
\label{eq:covnum}
\end{equation}
A forecast can only covary with the target through the predictable part. No
amount of skill lets it reach $\eps$.

\emph{Step 2.} The law of total variance gives
$\Var(y)=\Var(\mu)+\E[\Var(y\mid\Ft)]$, so the target's spread is the
predictable spread plus the average conditional spread. Only the first is
available for a forecast to match.

\emph{Step 3.} Dividing \eqref{eq:covnum} by the two standard deviations
and inserting $\mathrm{sd}(\mu)$ top and bottom separates the two effects:
\begin{equation}
\begin{aligned}
  \IC
  &=\frac{\Cov(\yhat,y)}{\mathrm{sd}(\yhat)\,\mathrm{sd}(y)}
   =\frac{\Cov(\yhat,\mu)}{\mathrm{sd}(\yhat)\,\mathrm{sd}(y)} \\
  &=\underbrace{\Corr(\yhat,\mu)}_{\rho}
    \cdot
    \underbrace{\frac{\mathrm{sd}(\mu)}{\mathrm{sd}(y)}}_{\IC_{\max}} .
\end{aligned}
\label{eq:ic-factorization}
\end{equation}
The first factor is forecast skill relative to the conditional mean, and a model
can move it. The second is fixed by the target and the declared information set,
and no model can.

\emph{Step 4.} Substituting Step 2 into
$\IC_{\max}$ and dividing numerator and denominator by $\E[\Var(y\mid\Ft)]$
turns the ceiling into the stated function of the signal-to-noise ratio:
\begin{equation}
  \IC_{\max}
  =\sqrt{\frac{\Var(\mu)}{\Var(\mu)+\E[\Var(y\mid\Ft)]}}
  =\sqrt{\frac{\SNR}{1+\SNR}} ,
  \label{eq:icmax-snr}
\end{equation}
so $\IC=\rho\sqrt{\SNR/(1+\SNR)}$. Under \ref{a:weak} the ceiling behaves like
$\sqrt{\SNR}$, and the square root matters: it leaves the ceiling far above the
signal-to-noise ratio itself, so $\SNR=0.01$ still permits
$\IC_{\max}\approx0.10$.
\end{proof}

The multiplicative factor is the reliability ratio of errors-in-variables
regression \citep{fuller1987measurement,carroll2006measurement}, with roots in
\citet{spearman1904proof}.

The amplitude identity is a statement about slopes through the origin, so it
needs one condition that the results above do not.

\begin{paperassumptions}[title={Assumption for \protect\Cref*{prop:collapse}}]
\begin{itemize}[leftmargin=1.6em,itemsep=1pt,topsep=2pt]
  \refstepcounter{assumption}\item[\textbf{\theassumption}]\label{a:center}%
    \emph{Centering:} variables are centered, so the rescaling is a slope
    through the origin.
\end{itemize}
\end{paperassumptions}

\begin{proof}[Proof of \Cref{prop:collapse}]
\emph{Step 1.} Expand the objective in $s$. Under
\ref{a:center} both variables are centered, so second moments are variances:
\begin{equation}
  \E\big[(s\yhat-y)^{2}\big]
  =s^{2}\Var(\yhat)-2s\Cov(\yhat,y)+\Var(y).
  \label{eq:rescale-risk}
\end{equation}
The right-hand side is a parabola in $s$ whose leading coefficient
$\Var(\yhat)$ is strictly positive, so it is convex and its stationary point is
the unique minimizer. Setting the derivative $2s\Var(\yhat)-2\Cov(\yhat,y)$ to
zero gives
\begin{equation}
  s^{\star}=\frac{\Cov(\yhat,y)}{\Var(\yhat)} .
  \label{eq:sstar}
\end{equation}
This is the least-squares slope of $y$ on $\yhat$, so $s^{\star}$ is the single
best post-hoc rescaling available to any user of the forecast: it is what a
practitioner would recover by regressing realized values on predictions.

\emph{Step 2.} Scaling multiplies the standard
deviation by $\lvert s^{\star}\rvert$, and one factor of $\mathrm{sd}(\yhat)$
then cancels against the variance:
\begin{equation}
\begin{aligned}
  \ampopt
  &=\frac{\mathrm{sd}(s^{\star}\yhat)}{\mathrm{sd}(y)}
   =\frac{\lvert s^{\star}\rvert\,\mathrm{sd}(\yhat)}{\mathrm{sd}(y)} \\
  &=\frac{\lvert\Cov(\yhat,y)\rvert}{\Var(\yhat)}
    \cdot\frac{\mathrm{sd}(\yhat)}{\mathrm{sd}(y)} \\
  &=\frac{\lvert\Cov(\yhat,y)\rvert}{\mathrm{sd}(\yhat)\,\mathrm{sd}(y)}
   =\lvert\Corr(\yhat,y)\rvert .
\end{aligned}
\label{eq:amplitude-proof}
\end{equation}
No model has been used: this holds for any centered pair with $\Var(\yhat)>0$.
The identity is worth stating in words. Once a forecast has been given its best
possible scaling, the amplitude it retains is not a separate property to be
tuned; it is exactly its correlation with the target. A forecast cannot be both
weakly correlated and correctly scaled.

\emph{Step 3.} Here the conditional-mean decomposition enters. By
\Cref{prop:ic} the correlation factors as $\Corr(\yhat,y)=\rho\,\IC_{\max}$, and
$\lvert\rho\rvert\le1$ because $\rho$ is itself a correlation. Substituting into
Step 2,
\begin{equation}
  \ampopt=\lvert\Corr(\yhat,y)\rvert=\lvert\rho\rvert\,\IC_{\max}\le\IC_{\max}.
  \label{eq:collapse-bound}
\end{equation}
The ceiling depends only on the target and the declared information set, so it
binds every admissible forecast at once. Improving the model moves
$\lvert\rho\rvert$ toward one and nothing else.
\end{proof}

\begin{remark}[title=The identity and the bound require different assumptions]
Step 2 is an identity about least squares and holds universally. Step 3 is where
predictability enters through the conditional-mean decomposition. The
best-scaled amplitude identity remains valid for any centered forecast-target
pair, while the ceiling depends on the declared information set.
\end{remark}

\begin{remark}[title=Centering and attribution]
The equality is the classical ordinary-least-squares slope identity written in
standard-deviation units \citep{fuller1987measurement,carroll2006measurement}.
With an intercept, or with variables centered as in \ref{a:center}, a constant
offset in $\delta$ is absorbed. A slope fit through the origin on uncentered
variables has a different minimizer. Thus a best-scaled centered forecast with
weak correlation necessarily has weak amplitude.
\end{remark}

\subsection{Mean and median point forecasts}

This subsection is the only place the conditional law is given a shape.

\begin{paperassumptions}[title={Assumption for \protect\Cref*{prop:quantile}}]
\begin{itemize}[leftmargin=1.6em,itemsep=1pt,topsep=2pt]
  \refstepcounter{assumption}\item[\textbf{\theassumption}]\label{a:symm}%
    \emph{Location family symmetric about zero, with finite mean and a common
    conditional scale.}
\end{itemize}
\end{paperassumptions}

\begin{proof}[Proof of \Cref{prop:quantile}]
\emph{Setup.} Under \ref{a:symm} the conditional law of $y$ given
$\Ft$ is a location-scale family, $y=\mu+\sigma Z$, where $Z$ is symmetric about
zero with finite mean and $\sigma>0$ is the common conditional scale. Both
candidate forecasts are read off this one law.

\emph{Step 1.} Symmetry means $Z$ and $-Z$ have the same law,
so $\E[Z]=\E[-Z]=-\E[Z]$, and the finite mean in \ref{a:symm} makes this
subtraction legitimate, forcing $\E[Z]=0$. Since $\mu$ and $\sigma$ are
$\Ft$-measurable, linearity gives $\E[y\mid\Ft]=\mu+\sigma\E[Z]=\mu$.

\emph{Step 2.} The same symmetry gives
$\Pr(Z\le0)=\Pr(Z\ge0)$, and these two probabilities sum to at least one because
together they cover the line, so each is at least $\tfrac12$. That is the
definition of zero being a median of $Z$. The map $z\mapsto\mu+\sigma z$ is
strictly increasing for $\sigma>0$, and a strictly increasing map carries medians
to medians, so $q_{1/2}(y\mid\Ft)=\mu+\sigma\cdot0=\mu$.

\emph{Step 3.} The two forecasts therefore coincide: both are the
same $\Ft$-measurable variable $\mu$. Nothing distinguishes a mean-optimal from a
median-optimal point forecast here, so choosing absolute error over squared error
cannot escape the conclusion. Applying \Cref{prop:collapse} to the centered pair
$(\mu,y)$ gives $\ampopt=\lvert\Corr(\mu,y)\rvert$.
\end{proof}

\begin{remark}[title=The bound is attained by these forecasts]
The forecast $\yhat=\mu$ has $\rho=\Corr(\mu,\mu)=1$, so \Cref{prop:ic} turns the
inequality of \Cref{prop:collapse} into the equality $\ampopt=\IC_{\max}$. The
best case for a point forecast under \ref{a:symm} is the correlation ceiling
itself, and \ref{a:weak} is what makes that ceiling small. The common
conditional scale in \ref{a:symm} does no work at $\tau=\tfrac12$, since a
series-varying $\sigma_{i}$ leaves both forecasts at $\mu_{i}$; it binds away from
the median.
\end{remark}

\begin{remark}[title=Off-median quantiles under a series-varying scale]
If the conditional scale varies across series, the location-scale form gives
$q_{\tau}(y_{i}\mid\Ft)=\mu_{i}+\sigma_{i}q_{\tau}(Z)$. For $\tau\neq\tfrac12$
the term $\sigma_{i}q_{\tau}(Z)$ is not constant across $i$, so the forecast is a
blend of $\mu_{i}$ and $\sigma_{i}$ and \Cref{prop:quantile} does not apply. This
is the same ingredient as \Cref{cor:mseranksvol}, where cross-sectional variation
in $\sigma_{i}$ drives a ranking. Its effect on ordering depends on the relative
dispersion and dependence of $\mu_i$ and $\sigma_i$.
\end{remark}

\subsection{Insensitivity to cross-series structure}

\begin{proof}[Proof of \Cref{prop:separable}]
\emph{Step 1.} A pointwise risk is a sum of
per-coordinate terms,
\begin{equation}
  R(\yhat)=\sum_{i,t}\E\big[\ell(\yhat_{it},y_{it})\big],
  \label{eq:pointwise-risk}
\end{equation}
and each summand is the expectation of a function of the single pair
$(\yhat_{it},y_{it})$. Whenever it exists it is therefore an integral of $\ell$
against $\mathrm{Law}(\yhat_{it},y_{it})$ and against nothing else. Summing over
$i$ and $t$ adds such terms without introducing any new dependence, so $R$ is a
functional of the collection of pair laws
$\{\mathrm{Law}(\yhat_{it},y_{it})\}$ alone. Squared error, absolute error,
quantile loss and every other per-observation loss enter here on the same
footing; the argument uses only the shape of \eqref{eq:pointwise-risk}.

\emph{Step 2.} Call a panel law $P'$ a \emph{recoupling} of $P$
when every pair $(\yhat_{it},y_{it})$ has the same law under both. Step 1 then
gives $R(P)=R(P')$ for any recoupling, because $R$ never queries the joint
behavior of two distinct coordinates. The recouplings of a given $P$ are exactly
the joint laws having those pair laws as prescribed margins. That set collapses
to a single point only in degenerate cases; in general it is large, and
\Cref{ex:coupling} constructs a one-parameter family of its members explicitly.
Risk is thus constant on a set of panel laws that is typically infinite.

\emph{Step 3.} Cross-sectional correlation is computed
from the two vectors $(\yhat_{it})_{i}$ and $(y_{it})_{i}$ observed at a single
timestamp, so its law depends on the joint distribution across $i$. In
particular it depends on the cross terms $\Cov(\yhat_{it},y_{jt})$ for $i\neq j$,
which describe how one series' forecast relates to another series' outcome, and
which no pair law records. Those are precisely the quantities a recoupling is
free to move while Step 2 holds $R$ fixed.

\emph{Step 4.} Combining the two, a recoupling can change $\IC_{t}$
while leaving $R$ unchanged, so equal risk does not determine cross-sectional
correlation and two forecasters with identical pointwise risk can rank the
cross-section differently. \Cref{ex:coupling} exhibits this concretely: the risk
stays at $1.40$ exactly while $\E[\IC_{t}]$ ranges over an order of magnitude.
\end{proof}

\begin{example}[Fixed risk, varying cross-sectional correlation]
\label{ex:coupling}
Take $N=2$ and let $(\yhat_{1},y_{1},\yhat_{2},y_{2})$ be centered joint Gaussian
with unit variances, $\Corr(\yhat_{i},y_{i})=r$ for both coordinates,
$\Corr(\yhat_{1},\yhat_{2})=\Corr(y_{1},y_{2})=0$, and one free cross parameter
$g=\Cov(\yhat_{1},y_{2})=\Cov(\yhat_{2},y_{1})$.

Each coordinate's joint law is $(\yhat_{i},y_{i})$ bivariate normal with
correlation $r$, and is therefore the same for every admissible $g$. Squared-error
risk depends only on that law, since
$\E[(\yhat_{i}-y_{i})^{2}]=2(1-r)$, so it is constant in $g$.

The cross-sectional correlation is not. With two series it is the sign of
$(\yhat_{1}-\yhat_{2})(y_{1}-y_{2})$, and
\begin{equation}
  \Corr(\yhat_{1}-\yhat_{2},\,y_{1}-y_{2})
  =\frac{2r-2g}{2}=r-g ,
  \label{eq:coupling-correlation}
\end{equation}
so $\E[\IC_{t}]$ moves with $g$ while the risk does not. At $r=0.3$, sweeping
$g$ over $\{-0.25,\,0,\,0.25\}$ leaves the squared-error risk at $2(1-r)=1.40$
exactly while $\E[\IC_{t}]=\tfrac{2}{\pi}\arcsin(r-g)$ takes the values $0.371$,
$0.194$, and $0.032$: an order of magnitude, at identical loss.
\end{example}

\begin{remark}[title=Interpretation]
With enough capacity, the population MSE minimizer is $\yhat=\mu$ and can order
the cross-section as well as the information set allows. The proposition
establishes invariance across equal-risk couplings: pointwise risk alone does
not select the one with better $\IC$. A cross-sectional term resolves that
indifference for its chosen functional; alternative objectives may do so with a
different number or form of terms.
\end{remark}

\subsection{The direction law}

The Gaussian direction law translates a correlation into the hit rate often
reported in financial forecasting.

\begin{lemma}[Gaussian orthant probability]
\label{lem:orthant}
For a standard bivariate normal pair $(U,V)$ with correlation $r$,
$\Pr(U>0,V>0)=\tfrac14+\tfrac{1}{2\pi}\arcsin r$.
\end{lemma}

This is Sheppard's orthant identity \citep{sheppard1899error}; see also
\citet{kruskal1958ordinal}. His rotational-invariance argument runs as follows.

\begin{proof}[Proof of \Cref{lem:orthant}]
\emph{Step 1.} Write
$V=rU+\sqrt{1-r^{2}}\,W$ with $(U,W)$ a pair of independent standard normals.
This reproduces the required moments, since $\Var(V)=r^{2}+(1-r^{2})=1$ and
$\Cov(U,V)=r\Var(U)=r$, so $(U,V)$ has the law in the statement and the
probability may be computed in the $(U,W)$ plane instead.

\emph{Step 2.} The law of $(U,W)$ is a standard
bivariate normal with independent coordinates, whose density depends only on
$\lVert(u,w)\rVert$. It is therefore rotationally invariant, and the probability
of any cone with apex at the origin is its opening angle divided by $2\pi$. The
problem reduces to measuring one angle.

\emph{Step 3.} In the $(U,W)$ plane the event
$\{U>0,\,V>0\}$ is the intersection of two half-planes through the origin, with
inward normals $n_{1}=(1,0)$ and $n_{2}=(r,\sqrt{1-r^{2}})$. Two half-planes
whose normals meet at angle $\theta$ cut out a wedge of angle $\pi-\theta$, and
here $\cos\theta=n_{1}\cdot n_{2}=r$, so $\theta=\arccos r$. Hence
\begin{equation}
  \Pr(U>0,V>0)
  =\frac{\pi-\arccos r}{2\pi}
  =\frac14+\frac{1}{2\pi}\arcsin r ,
  \label{eq:orthant-angle}
\end{equation}
using $\arccos r=\tfrac{\pi}{2}-\arcsin r$.

\emph{Step 4.} Three reference values confirm the sign
convention: $r=0$ returns the quadrant probability $\tfrac14$, $r=1$ returns
$\tfrac12$ because the two events coincide, and $r=-1$ returns $0$ because they
are disjoint.
\end{proof}

Turning the lemma into a statement about a forecast and its target requires
distributional structure that nothing earlier in the section assumed.

\begin{paperassumptions}[title={Assumption for \protect\Cref*{prop:arcsin}}]
\begin{itemize}[leftmargin=1.6em,itemsep=1pt,topsep=2pt]
  \refstepcounter{assumption}\item[\textbf{\theassumption}]\label{a:gauss}%
    \emph{Joint Gaussianity} of the forecast and the target, used only for the
    exact direction law.
\end{itemize}
\end{paperassumptions}

\begin{proposition}[Direction law]
\label{prop:arcsin}
Under \ref{a:gauss} with zero means,
$\Pr(\mathrm{sign}\,\yhat=\mathrm{sign}\,y)=\tfrac12+\tfrac{1}{\pi}\arcsin(\IC)$.
\end{proposition}

\begin{proof}
\emph{Step 1.} Dividing by a positive constant leaves signs alone, so
$\mathrm{sign}\,\yhat=\mathrm{sign}\,(\yhat/\mathrm{sd}(\yhat))$ and likewise for
$y$. Replace the pair by $(\yhat/\mathrm{sd}(\yhat),\,y/\mathrm{sd}(y))$, which
under \ref{a:gauss} with zero means is a standard bivariate normal pair whose
correlation is $\IC$. Nothing about the event has changed.

\emph{Step 2.} Signs agree exactly on
$\{\yhat>0,y>0\}\cup\{\yhat<0,y<0\}$, a disjoint union. The ties
$\{\yhat=0\}\cup\{y=0\}$ have probability zero, since each is a line in a plane
carrying a density, so they may be ignored.

\emph{Step 3.} A centered Gaussian law is invariant under
$z\mapsto-z$, which maps the second event onto the first, so the two carry equal
probability and the total is twice either one. \Cref{lem:orthant} with $r=\IC$
then evaluates it:
\begin{equation}
  \Pr(\mathrm{sign}\,\yhat=\mathrm{sign}\,y)
  =2\Pr(\yhat>0,y>0)
  =\frac12+\frac{1}{\pi}\arcsin\IC .
  \label{eq:direction-law}
\end{equation}
\end{proof}

At $\IC=0.05$ this is a hit rate of $51.6\%$, which is the sense in which a
correlation that looks negligible is still economically meaningful, and equally
the sense in which a hit rate near one half is consistent with real skill.

\FloatBarrier
\section{Calibrated Cross-Sectional Ranking}
\label{sec:method}

Proposition~\ref{prop:separable} motivates scoring the cross-sectional
functional directly while retaining a calibration anchor. Our version,
\method, combines squared error with differentiable per-timestamp Pearson
correlation:
\begin{equation}
  \mathcal{L}
  =\frac{1}{M}\sum_{i,t}(\yhat_{it}-y_{it})^{2}
   -\frac{\lambda}{|B|}\sum_{t\in B}\widehat{\IC}_{t}.
  \label{eq:loss}
\end{equation}
The first term fixes scale; the second rewards ordering at each timestamp.
Other calibration and ranking surrogates fit the same construction, as detailed
in Section~\ref{sec:design}. Sweeping $\lambda$ traces the
calibration-ranking frontier in Figure~1(b).

\subsection{Relation to Existing Objectives}
\label{sec:design}

\Cref{prop:separable} motivates a two-term objective while leaving the
specific choice open. Writing the objective as
$\mathcal{L}=\mathcal{C}(\yhat,y;w)-\lambda\,\mathcal{S}(\yhat,y;w)$, there are
four independent axes. \Cref{tab:design} places published alternatives and the
evaluated \method variants in this space.

\begin{table}[!htbp]
  \centering
  \small
  \begin{tabularx}{\textwidth}{@{}l l >{\raggedright\arraybackslash}X l@{}}
    \toprule
    \tablehead{Axis} & \tablehead{Choices} & \tablehead{Prior art}
      & \tablehead{Ours} \\
    \midrule
    Calibration $\mathcal{C}$ & Squared error, absolute error, Huber, quantile
      & \citet{feng2019temporal} squared error
      & Squared error \\
    structure $\mathcal{S}$ & Pearson, Spearman, pairwise hinge, listwise
      & \citet{feng2019temporal} hinge; \citet{lin2026lambdarankic},
        \citet{yang2020qlib} rank correlation
      & differentiable Pearson \\
    Weight $w$ & Uniform, theory-derived $c_{t}$, learned, inverse-variance
      & \citet{kendall2018multitask} learned
      & Uniform; $c_{t}$ variant \\
    Tradeoff $\lambda$ & Fixed, swept
      & \citet{feng2019temporal} fixed
      & Swept, 7 points \\
    \bottomrule
  \end{tabularx}
  \caption{Calibration-ranking objective design space.}
  \label{tab:design}
\end{table}

The endpoints of the $\lambda$ range correspond to the two single-term
objectives, which are trained separately and reported as comparators.
Huber calibration is a natural extension because robust fidelity terms can
reduce the influence of large target innovations.

\FloatBarrier
\section{Experimental Design}
\label{sec:experiments}

\subsection{\dataset}
\label{subsec:finance1k}

\dataset contains hourly observations for 1{,}000 US equities over $28{,}510$
steps running from 2015 to early 2026. Its two aligned targets are the closing-price
log return, $\log(p_t/p_{t-1})$, and the log-volume change,
$\log(v_t/v_{t-1})$. The split is chronological: the first $19{,}957$ steps,
ending in late 2022, form the training set, and the following $8{,}553$ form the
test set. From the previous $96$ hourly observations of one series, the task is
to predict its target at the next step. Evaluation uses every fourth test step.
The calendar grid contains 684 weekend and holiday evaluation steps with zero
cross-sectional spread; all reported IC values use the remaining 1{,}455
trading hours.

\subsection{Forecasters and matched training protocol}
\label{subsec:forecaster}

The controlled reference encoder maps a length-$L{=}96$ context window of one
series to a next-step scalar forecast for that same series: the window is split
into $6$ patches of length $16$, each linearly embedded to width $d{=}128$ with a learned positional
term, encoded by a $3$-layer Transformer ($4$ heads, Gaussian error linear unit
activations, dropout $0.1$), mean-pooled over patches, and passed through a
LayerNorm$+$linear head ($\approx0.38$M parameters). In the weighted \method
variant, the scalar proxy $c_t$ is concatenated to the pooled representation
before the head. This controlled encoder is used only to isolate objective
effects and is not proposed as a new forecasting architecture. Both loss terms are weighted by
$c_t=\sqrt{\widehat{\SNR}_t/(1+\widehat{\SNR}_t)}$, computed from a 24-hour
trailing volatility window and normalized on training statistics. The main
controlled comparison uses the uniform objective in \eqref{eq:loss} without this
input. A matched ablation finds no detectable contribution from the weighting.

For the twelve standard trainable families, a shared two-layer scalar mixer applies
the same conditioning without modifying the backbone. All supervised models
receive the same univariate history for each series and share parameters across
series. We train with AdamW at learning rate $3\times10^{-4}$, weight decay
$10^{-4}$, batch size 24 complete cross-sections, and 1{,}500 steps over seeds
11, 23, and 42. The objective is the only change in the controlled comparison.
Foundation-model fine-tuning uses learning rate $10^{-6}$ and one timestamp per
batch.

We report squared error, raw amplitude from \eqref{eq:raw-amplitude}, and
Pearson IC computed within each trading hour and then averaged over hours.
Computing IC after pooling all $(i,t)$ pairs would mix temporal and
cross-sectional variation and answer a different question.

\subsection{Synthetic verification design}
\label{subsec:synthetic-design}

The real targets do not reveal their conditional means, so neither best-scaled
amplitude nor the predictability ceiling can be observed directly. We therefore
generate panels with known signal, noise, and model error. The sweep spans two
decades of signal-to-noise ratio and three levels of forecaster skill, with
fitting and evaluation on disjoint draws. Section~\ref{sec:results} reports the
quantitative verification.

\paragraph{Generator specification.}
\label{subsec:generator}
For $N$ series over $T$ time steps, under the common-SNR design, the
per-series noise scale is lognormal, the signal variance scales with it, and the model-error scale
sets $Q_{i}$ and $\rho_{i}$:
\begin{align}
  \sigma_{i} &= \exp\!\big(\log\sigma_{\mathrm{med}}
    + \mathcal{N}(0,\sigma_{\mathrm{disp}}^{2})\big), &
  v_{i} &= \SNR\cdot\sigma_{i}^{2},
  \label{eq:generator-scale} \\
  \mu_{ti} &= \mathcal{N}(0,1)\cdot\sqrt{v_{i}}, &
  \eps_{ti} &= \mathcal{N}(0,1)\cdot\sigma_{i},
  \label{eq:generator-signal-noise} \\
  \tau_{i} &= \sqrt{q_{\mathrm{frac}}\,v_{i}}\cdot
    \exp\!\big(\mathcal{N}(0,q_{\mathrm{disp}}^{2})\big), &
  \delta_{ti} &= \mathcal{N}(0,1)\cdot\tau_{i},
  \label{eq:generator-model-error}
\end{align}
with $y=\mu+\eps$, $\yhat=\mu+\delta$, and $\delta\perp\eps$
(satisfying Assumption~\ref{a:noleak}). Then $Q_{i}=\tau_{i}^{2}$ and
$\rho_{i}=\sqrt{v_{i}/(v_{i}+\tau_{i}^{2})}$. Each experiment draws an
independent random number generator (RNG) stream keyed by experiment name
from the global seed~$11$, so results are reproducible regardless of call order.

\paragraph{Synthetic predictor.}
Each experiment uses $\yhat=\mu+\delta$, built by adding controlled
independent error $\delta$ to the true signal $\mu$, so its signal alignment
$\rho$ is set by the single knob $q_{\mathrm{frac}}$ (with $\rho=1$ at
$q_{\mathrm{frac}}=0$). Known skill and noise make the propositions directly
testable against ground truth. Real-model results appear in
\Cref{sec:comprehensive}.

\paragraph{Locked synthetic parameters.}
\label{subsec:synthetic-params}
\begin{figure}[!t]
  \centering
  \includegraphics[width=\textwidth]{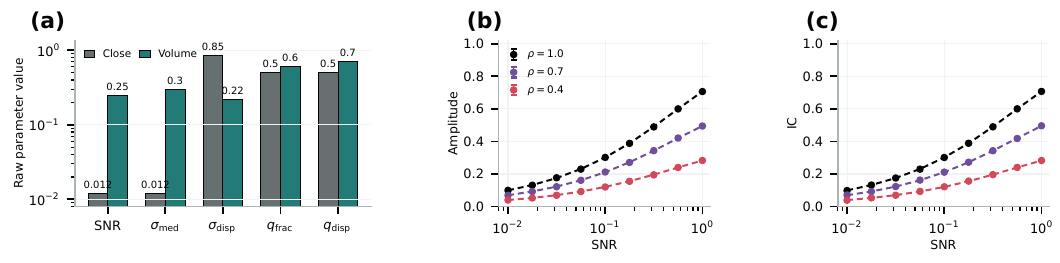}
  \caption{\textbf{Synthetic verification of the predictability mechanism.}
  \textbf{(a)} Parameters for the low-SNR close and higher-SNR volume regimes.
  \textbf{(b)} Raw amplitude and \textbf{(c)} cross-sectional IC across SNR and
  forecaster-skill levels. Markers show three-seed means, bars one standard
  deviation, and dashed curves the analytical predictions.}
  \label{fig:section7-evidence}
\end{figure}

\Cref{fig:section7-evidence}(a) sets close as the low-SNR regime: its signal variance is only
$1.2\%$ of the noise variance and its larger dispersion creates heterogeneous
noise floors across series. Volume has a $20\times$ larger SNR and more
homogeneous noise. This contrast isolates the predicted transition from
MSE-IC decoupling to partial coupling without changing the generator or
evaluation metrics.

\FloatBarrier
\subsection{Additional evaluation protocols}
\label{subsec:additional-evaluation}

\paragraph{Held-out decision return.}
The training objective uses Pearson IC, so we also evaluate a downstream
functional that is absent from the loss. At each timestamp we subtract the
cross-sectional mean from the forecast, so the resulting weights sum to zero,
and divide by their total absolute size, so they sum to one in magnitude.
Applying those weights to the realized returns gives a return per timestamp. The
reported statistic is its mean divided by its standard deviation, the Sharpe
ratio of that decision rule. These archived runs use the
weighted controlled-encoder protocol shared by the cross-architecture experiment.

\paragraph{Post-hoc amplitude sweep.}
The headline controlled comparison uses the uniform objective at the
prespecified $\lambda=1$. Separately, we examine the earlier weighted-objective
variant across six values of $\lambda$ on the evaluation steps. The sweep is
post hoc and does not select the headline test result.

\paragraph{Computing infrastructure.}
Every run executes on a single NVIDIA RTX Pro 6000 with one visible CUDA device,
and the OpenMP, MKL, OpenBLAS, and NumExpr thread counts are pinned to one before
any numerical import so that a seeded run repeats exactly.

\FloatBarrier
\section{Empirical Results}
\label{sec:results}

We first isolate the objective on the controlled encoder, then test the
amplitude identity synthetically and across public foundation-model benchmarks.
The final comparisons examine a held-out decision functional, the weighting
ablation, and the high-predictability volume target. Comprehensive
cross-architecture evidence follows in \Cref{sec:comprehensive}.

\begin{figure}[!t]
  \centering
  \includegraphics[width=\textwidth]{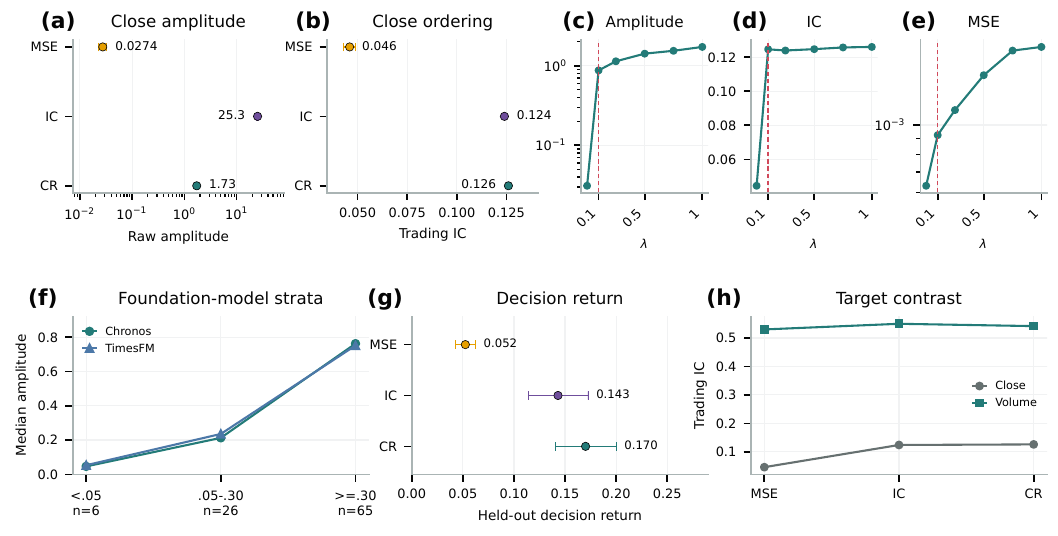}
  \caption{\textbf{Empirical evidence for forecast collapse and \method.}
  \textbf{(a,b)} Raw amplitude and cross-sectional IC under MSE, IC, and
  \method. \textbf{(c--e)} The calibration-ranking tradeoff as $\lambda$
  varies. \textbf{(f)} TSFM amplitude across achieved-$R^{2}$ strata.
  \textbf{(g)} Held-out decision return. \textbf{(h)} Return against volume as
  a target contrast. Error bars are one standard deviation over three seeds
  where available.}
  \label{fig:section8-evidence}
\end{figure}

\subsection{Controlled objective comparison on \dataset}

\begin{table}[htbp]
  \centering
  \small
  \begin{tabular}{llccc}
    \toprule
    \tablehead{Target} & \tablehead{Objective} & \tablehead{$\ampraw\to1$} & \tablehead{$\IC\uparrow$} & \tablehead{$\MSE\downarrow$} \\
    \midrule
    Close return & MSE & $0.027\mathbin{\pm}0.005$ & $0.046\mathbin{\pm}0.003$ & $(4.30\mathbin{\pm}0.13)\!\times10^{-4}$ \\
     & IC & $25.260\mathbin{\pm}1.114$ & $0.124\mathbin{\pm}0.000$ & $(4.22\mathbin{\pm}2.84)\!\times10^{-1}$ \\
     & \method & $1.836\mathbin{\pm}0.118$ & $0.126\mathbin{\pm}0.002$ & $(1.43\mathbin{\pm}0.06)\!\times10^{-3}$ \\
     & \method (weighted) & $1.726\mathbin{\pm}0.106$ & $0.126\mathbin{\pm}0.000$ & $(1.29\mathbin{\pm}0.12)\!\times10^{-3}$ \\
    \midrule
    Volume change & MSE & $0.391\mathbin{\pm}0.011$ & $0.530\mathbin{\pm}0.003$ & $(3.19\mathbin{\pm}0.02)\!\times10^{0}$ \\
     & IC & $0.233\mathbin{\pm}0.013$ & $0.550\mathbin{\pm}0.004$ & $(7.08\mathbin{\pm}0.25)\!\times10^{0}$ \\
     & \method & $0.391\mathbin{\pm}0.009$ & $0.538\mathbin{\pm}0.002$ & $(3.18\mathbin{\pm}0.01)\!\times10^{0}$ \\
     & \method (weighted) & $0.393\mathbin{\pm}0.006$ & $0.542\mathbin{\pm}0.002$ & $(3.14\mathbin{\pm}0.01)\!\times10^{0}$ \\
    \bottomrule
  \end{tabular}
  \caption{\textbf{Objective comparison on the reference backbone.}
    Raw amplitude and IC are stated on trading rows, squared error on
    the full evaluation grid. \method is the uniform objective of
    \eqref{eq:loss} at $\lambda=1$; the weighted row is the
    per-timestamp weighted variant. Mean $\mathbin{\pm}$ one standard
    deviation over three seeds.}
  \label{tab:reference-objectives}
\end{table}

\Cref{tab:reference-objectives} and \Cref{fig:section8-evidence}(a,b) expose both halves of the result. MSE produces a nearly flat
forecast, at $0.027$ of target amplitude and IC $0.046$. IC training raises IC to
$0.124$ but leaves scale unidentified, reaching $25\times$ the target amplitude.
Uniform \method at
the prespecified $\lambda=1$ reaches the same ordering, IC $0.126$, while keeping raw
amplitude within a factor of two of the target. Pure ranking incurs over two
orders of magnitude more squared error than \method.

For the MSE row, the ratio of the reported mean IC to mean raw amplitude is
$0.046/0.027=1.68$. This is a descriptive scale gap rather than
$s_t^\star$: the plotted entries average timestamp-level quantities, whereas
$s_t^\star$ is a ratio of covariance to variance at each timestamp. Under a
common-slope approximation, the ratio says that an additional $1.7\times$
rescaling would be optimal. The observed collapse therefore includes finite-fit
scale error beyond the best-scaled attenuation in Proposition~\ref{prop:collapse}.

These objectives answer different questions. MSE shows the forecast selected by a
calibration objective when predictable variance is small. IC verifies that the
architecture and inputs can recover substantially more ordering, while its raw
amplitude demonstrates the scale indeterminacy of correlation. The composite
shows that the ordering gain does not require accepting that indeterminacy. It
is an operating point on a frontier; the downstream cost of amplitude error
determines which point is useful.

\FloatBarrier
\subsection{The calibration-ranking frontier}

The seven nonnegative values of $\lambda$ in Figure~1(b) are evaluated
at three seeds. For the cross-architecture records, IC is restated on trading hours while MSE remains on the full evaluation grid; MSE has no uniform
restatement factor. The first nonzero ranking weight produces most of the IC
gain; larger weights mainly increase amplitude and squared error. The
single-term MSE and IC objectives are trained separately and appear at the two
extremes. The full MSE range is retained on the log-scaled axis, including the
high-error IC-only points on the right. The red line connects the empirically
non-dominated points across all backbone objectives. The cyan line
traces the seven-point Autoformer sweep, while the remaining points show matched
MSE and \method runs for the other eleven backbones and the zero-shot or
fine-tuned foundation-model baselines. Red therefore identifies the best
observed calibration-ranking tradeoffs regardless of architecture; cyan
holds the architecture fixed and isolates movement caused by the objective.
Autoformer is highlighted because its seven mean MSE values are strictly ordered
and span $3.05\times$, keeping all operating points legible. The
twelve-backbone endpoint movements appear in \Cref{tab:arch-ablation-sweep};
the plotted sweep includes all seven points and their three-seed uncertainty.
All four red points come from Non-stationary Transformer. Their near-vertical
shape is consistent with its stationarization preserving forecast scale while
the ranking term changes cross-sectional order.

\paragraph{Post-hoc weighted-objective sweep.}
The headline controlled comparison uses the uniform objective at the
prespecified $\lambda=1$. Separately, we examine the earlier weighted-objective
variant across six values of $\lambda$ on the evaluation steps.
\Cref{fig:section8-evidence}(c--e) reports trading-hour raw amplitude, IC,
and squared error,
averaged over three seeds. This sweep is post hoc and does not select the
headline test result.

\FloatBarrier

IC reaches $0.1246$ by $\lambda=0.1$ and changes by less than $0.002$ through
$\lambda=1$, while raw amplitude continues to rise. At $\lambda=0.1$, the
forecast reaches $87.5\%$ of target amplitude with about half the squared error
of the $\lambda=1$ setting. The measured points bracket unit amplitude between
$\lambda=0.1$ and $0.25$; linear interpolation places the crossing near
$\lambda=0.17$. This descriptive crossing illustrates the validation criterion
in \Cref{sec:discussion}; selecting an operating point for deployment still
requires a validation split taken by date.
\Cref{fig:section8-evidence}(c--e) makes this asymmetry visible directly: IC is nearly horizontal
after the first nonzero weight, whereas amplitude and MSE continue to grow.

\subsection{Synthetic verification}

Across the synthetic grid, \eqref{eq:collapse} predicts held-out best-scaled
amplitude to within $0.0015$ and correlation to within $0.0021$. This checks the
quantitative identity under controlled conditions; the GIFT-Eval scan in
\Cref{fig:section8-evidence}(f) supplies complementary evidence from pretrained models on real
benchmark data.

\Cref{fig:section7-evidence}(b,c) plots the sweep, which spans two decades of signal-to-noise
ratio and three levels of forecaster skill. It is fitted and evaluated on
disjoint halves, so the amplitude ratio is a genuine out-of-sample measurement
rather than an in-sample slope. At the low end of the sweep even a perfect
ranker retains under a tenth of the target's amplitude.

\FloatBarrier
\subsection{Foundation-model evidence}

TimesFM and Chronos are evaluated zero-shot on Finance1K and across the 97
GIFT-Eval configurations. The benchmark analysis uses each model's own point
forecast after one fixed stationarization rule. We compare raw amplitude with the
$R^{2}$ achieved by a fixed set of fitted baselines on the same configuration.
Because that set may miss available structure, a large value shows that a
target is predictable, while a small value cannot show it is unpredictable.
The association in \Cref{fig:section8-evidence}(f) is therefore a cross-dataset regularity; the
exact bound remains a statement about best-scaled forecasts under a declared
information set.

For an empirical summary, we divide achieved $R^{2}$ into three descriptive
strata using the conventions $0.30$ and $0.05$. \Cref{fig:section8-evidence}(f)
reports their median raw amplitude across all 97 observations.

\paragraph{Raw amplitude follows achieved predictability empirically.}
We scan all 97 GIFT-Eval configurations \citep{aksu2024gift} under a fixed
stationarization rule. The fitted baselines attain median $R^{2}=0.56$, while the
same procedure attains $0.0044$ on Finance1K close returns. These achieved
values lower-bound the predictability available to a richer forecaster; a small
value therefore does not prove the ceiling is low.

\Cref{fig:section8-evidence}(f) compares those achieved values with the raw amplitude output
by two pretrained models. Amplitude correlates with $\sqrt{R^{2}}$ at $0.88$
for Chronos and $0.87$ for TimesFM. In the 65 configurations with achieved
$R^{2}\ge0.30$, both models retain about three quarters of target amplitude.
In the six configurations with achieved $R^{2}<0.05$, they retain under six
percent. The cutoffs are descriptive strata, not derived thresholds. This
result establishes a strong empirical association. It does not test the upper
bound in Proposition~\ref{prop:collapse}.

\FloatBarrier
\subsection{Decision metric, ablation, and negative control}

As an out-of-objective check, we turn each forecast into a decision: subtract
the cross-sectional mean so the weights sum to zero, scale them to unit total
size, and apply them to the realized returns. In the archived matched
controlled runs, the held-out decision return rises from
$0.052\pm0.010$ under MSE to $0.143\pm0.029$ under IC and
$0.170\pm0.030$ under weighted \method. This functional is absent from
\eqref{eq:loss}; its definition is given in
\Cref{subsec:additional-evaluation}. The separation in
\Cref{fig:section8-evidence}(g) is important because this return is not
optimized directly: the composite objective improves the downstream functional after
the loss has already balanced amplitude and ordering.

\FloatBarrier

A matched ablation replaces the theory-motivated per-timestamp weighting in the
original implementation with uniform weights. Weighted minus uniform IC is
$-0.0006$ on close and $+0.0014$ on volume, against weighted-arm seed ranges
of $0.0019$ and $0.0050$, respectively; raw amplitude is similar. The gain
therefore comes from scoring cross-sectional structure. We use the simpler
uniform objective in \eqref{eq:loss}; the complete cross-architecture ablation
appears in \Cref{sec:comprehensive}.

\paragraph{Negative control.}
\Cref{fig:section8-evidence}(h) shows the target contrast. On volume changes
from the same panel, MSE reaches
$(\ampraw,\IC)=(0.391,0.530)$ and uniform \method reaches $(0.391,0.538)$
(\Cref{tab:reference-objectives}).
The near-identical amplitude and modest IC difference contrast with close
returns under the same split and protocol. When more variation is predictable,
MSE does not produce the near-zero raw amplitude seen in \Cref{fig:section8-evidence}(a).
Because MSE already recovers strong ordering on this target, the cross-sectional
term has little remaining IC to recover; the small gain is the expected
negative-control result.

\FloatBarrier
\section{Comprehensive Cross-Architecture Evidence}
\label{sec:comprehensive}

This section reports every aggregate in the unified empirical protocol.
All supervised entries use seeds $11,23,42$ and report mean $\pm$ one
standard deviation. \Cref{tab:arch} compares the three training objectives,
while \Cref{tab:arch-ablation-sweep} separates the conditioning $\Delta$IC
contrast from the $\lambda=0$ and $1$ endpoints. Together, the two tables and
Figure~1(b) report the objective comparison, the complete conditioning matrix,
and the seven-point sweeps over all twelve backbones.

\begin{table}[!t]
  \belowrulesep=2pt
  \aboverulesep=0.5pt
  \centering
  \caption{\textbf{Objective comparison across twelve forecasting
    models.} Mean $\mathbin{\pm}$ one standard deviation over three
    seeds. IC is evaluated on trading hours; MSE uses the full
    evaluation grid.}
  \label{tab:arch}
  \setlength{\tabcolsep}{5pt}
  \renewcommand{\arraystretch}{1.18}
  \resizebox{\linewidth}{!}{%
  \begin{tabular}{@{}l cc @{\hspace{8pt}} cc @{\hspace{8pt}} cc@{}}
    \toprule
    \multicolumn{1}{c}{\multirow{2}[0]{*}{\tablehead{Backbone}}} & \multicolumn{2}{c}{\tablehead{MSE loss}} & \multicolumn{2}{c}{\tablehead{IC loss}} & \multicolumn{2}{c}{\tablehead{CalibRank}} \\
    \cmidrule(lr){2-3}\cmidrule(lr){4-5}\cmidrule(l){6-7}
    & \tablehead{$\IC\uparrow$} & \tablehead{$\MSE\downarrow$} & \tablehead{$\IC\uparrow$} & \tablehead{$\MSE\downarrow$} & \tablehead{$\IC\uparrow$} & \tablehead{$\MSE\downarrow$} \\
    \midrule
    DLinear~\citep{zeng2023transformers} & $0.096\mathbin{\pm}0.024$ & $(4.12\mathbin{\pm}0.03)\!\times10^{-4}$ & $0.129\mathbin{\pm}0.000$ & $(1.82\mathbin{\pm}1.58)\!\times10^{-2}$ & $0.127\mathbin{\pm}0.004$ & $(5.20\mathbin{\pm}1.84)\!\times10^{-4}$ \\
    PatchTST~\citep{nie2022patchtst} & $0.082\mathbin{\pm}0.016$ & $(4.79\mathbin{\pm}1.10)\!\times10^{-4}$ & $0.123\mathbin{\pm}0.007$ & $(4.41\mathbin{\pm}0.85)\!\times10^{-1}$ & $0.122\mathbin{\pm}0.008$ & $(5.02\mathbin{\pm}0.32)\!\times10^{-4}$ \\
    iTransformer~\citep{liu2023itransformer} & $0.091\mathbin{\pm}0.002$ & $(4.33\mathbin{\pm}0.08)\!\times10^{-4}$ & $0.129\mathbin{\pm}0.002$ & $(2.35\mathbin{\pm}0.38)\!\times10^{-1}$ & $0.127\mathbin{\pm}0.004$ & $(5.53\mathbin{\pm}0.60)\!\times10^{-4}$ \\
    TimeMixer~\citep{wang2024timemixer} & $0.110\mathbin{\pm}0.003$ & $(4.14\mathbin{\pm}0.03)\!\times10^{-4}$ & $0.129\mathbin{\pm}0.000$ & $(1.07\mathbin{\pm}0.22)\!\times10^{-2}$ & $0.127\mathbin{\pm}0.003$ & $(5.24\mathbin{\pm}0.81)\!\times10^{-4}$ \\
    Autoformer~\citep{wu2021autoformer} & $0.021\mathbin{\pm}0.007$ & $(4.30\mathbin{\pm}0.09)\!\times10^{-4}$ & $0.096\mathbin{\pm}0.008$ & $(3.58\mathbin{\pm}0.63)\!\times10^{-1}$ & $0.105\mathbin{\pm}0.003$ & $(7.86\mathbin{\pm}2.31)\!\times10^{-4}$ \\
    Transformer~\citep{vaswani2017attention} & $-0.011\mathbin{\pm}0.023$ & $(4.26\mathbin{\pm}0.14)\!\times10^{-4}$ & $0.121\mathbin{\pm}0.001$ & $(4.45\mathbin{\pm}0.47)\!\times10^{-1}$ & $0.120\mathbin{\pm}0.002$ & $(5.91\mathbin{\pm}0.50)\!\times10^{-4}$ \\
    Informer~\citep{zhou2021informer} & $-0.004\mathbin{\pm}0.017$ & $(4.18\mathbin{\pm}0.02)\!\times10^{-4}$ & $0.117\mathbin{\pm}0.002$ & $(4.64\mathbin{\pm}3.91)\!\times10^{-1}$ & $0.105\mathbin{\pm}0.005$ & $(8.26\mathbin{\pm}1.17)\!\times10^{-4}$ \\
    FEDformer~\citep{zhou2022fedformer} & $0.024\mathbin{\pm}0.013$ & $(5.43\mathbin{\pm}1.19)\!\times10^{-4}$ & $0.120\mathbin{\pm}0.001$ & $(1.18\mathbin{\pm}0.95)\!\times10^{0}$ & $0.117\mathbin{\pm}0.004$ & $(5.77\mathbin{\pm}1.89)\!\times10^{-4}$ \\
    Nonstat.\ Transformer~\citep{liu2022nonstationary} & $0.118\mathbin{\pm}0.008$ & $(4.09\mathbin{\pm}0.01)\!\times10^{-4}$ & $0.148\mathbin{\pm}0.001$ & $(5.27\mathbin{\pm}0.14)\!\times10^{-4}$ & $0.147\mathbin{\pm}0.002$ & $(4.26\mathbin{\pm}0.28)\!\times10^{-4}$ \\
    Reformer~\citep{kitaev2020reformer} & $0.062\mathbin{\pm}0.029$ & $(4.21\mathbin{\pm}0.10)\!\times10^{-4}$ & $0.119\mathbin{\pm}0.001$ & $(1.50\mathbin{\pm}0.28)\!\times10^{-1}$ & $0.121\mathbin{\pm}0.002$ & $(5.82\mathbin{\pm}1.17)\!\times10^{-4}$ \\
    LightTS~\citep{zhang2022lightts} & $0.092\mathbin{\pm}0.009$ & $(4.12\mathbin{\pm}0.00)\!\times10^{-4}$ & $0.128\mathbin{\pm}0.001$ & $(2.63\mathbin{\pm}4.43)\!\times10^{0}$ & $0.128\mathbin{\pm}0.001$ & $(5.03\mathbin{\pm}1.04)\!\times10^{-4}$ \\
    TSMixer~\citep{chen2023tsmixer} & $0.062\mathbin{\pm}0.010$ & $(4.15\mathbin{\pm}0.03)\!\times10^{-4}$ & $0.131\mathbin{\pm}0.000$ & $(4.11\mathbin{\pm}1.13)\!\times10^{-3}$ & $0.130\mathbin{\pm}0.001$ & $(5.35\mathbin{\pm}1.43)\!\times10^{-4}$ \\
    \bottomrule
  \end{tabular}%
  }
\end{table}

\subsection{Matched objectives across backbones}

Figure~1(b) distinguishes the red empirical Pareto frontier from the cyan
Autoformer sweep. The red line reports the best observed tradeoffs over the full
backbone set. The cyan path keeps the architecture fixed, so its
movement isolates the effect of changing $\lambda$. We repeat the matched MSE
and \method comparison across recurrent, convolutional, mixing, linear, and
attention-based architectures under the same context, optimizer, schedule, and
three-seed protocol. IC values in \Cref{tab:arch} use trading hours; MSE values use the full evaluation grid, matching the archived
cross-architecture records.

\Cref{tab:arch} reports the result. Under squared error the
backbones span a wide range of IC, but replacing MSE with \method improves every
one. The mean IC rises from $0.062$ to $0.123$, while mean squared error changes
from $4.3\times10^{-4}$ to $5.8\times10^{-4}$. The latter is over 800 times
smaller than the IC-only mean of $4.9\times10^{-1}$. The consistent
within-backbone gain makes architecture an unlikely explanation for the
ordering failure.

\subsection{Conditioning ablation and target contrast}

\paragraph{What the ablation shows.}
\Cref{tab:arch-ablation-sweep} reports conditioning-minus-no-conditioning
$\Delta$IC for all twelve backbones and both targets, together with every
$\Delta$MSE and seed standard deviation. The close deltas straddle zero,
while volume deltas are more often favorable. This is expected rather than
contradictory: hourly close volatility has little lagged predictability, so its
proxy input is often noisy; volume persistence makes the same input more
informative. The matched comparison separates this conditioning limitation from
the core \method{} coupling, whose weighted objective is unchanged in the
ablation.

\paragraph{Why volume looks different from close.}
The held-out prediction trajectories show that the MSE forecasts have more
visible variation than on close because volume log-change
has substantially higher SNR: its conditional mean occupies a larger fraction of realized
variance. Pure IC predictions still show avoidable scale drift because
correlation does not identify amplitude. \method tracks the realized scale more
closely by retaining a calibration anchor, and the seed ribbons show that the
distinction is stable across initializations rather than caused by a selected
trajectory.

\subsubsection{No-conditioning ablation by model and target}

For each backbone, we compute the mean and standard deviation of \method{} minus
its matched no-conditioning run. Positive $\Delta$IC and negative $\Delta$MSE
are favorable. The aggregate pattern is that the conditioning channel raises close IC
in only 4/12 backbones, but raises volume IC in 10/12 and lowers volume MSE in
9/12. Several close MSE improvements coexist with negligible or negative IC
changes because calibration can benefit from scale information even when that
information does not locate cross-sectional ranking skill. Large standard
deviations for Informer and Autoformer caution against interpreting isolated
signs as universal gains.

\begin{table}[!t]
  \belowrulesep=2pt
  \aboverulesep=0.5pt
  \centering
  \caption{\textbf{Conditioning ablation and calibration-ranking
    sweep.} Conditioning columns report CalibRank minus
    no-conditioning $\Delta$IC; sweep columns report the close-target
    endpoints at $\lambda=0$ and $\lambda=1$. Mean $\mathbin{\pm}$ one
    standard deviation over three seeds.}
  \label{tab:arch-ablation-sweep}
  \setlength{\tabcolsep}{5pt}
  \renewcommand{\arraystretch}{1.18}
  \resizebox{\linewidth}{!}{%
  \begin{tabular}{@{}l cc @{\hspace{10pt}} cc cc@{}}
    \toprule
    \multicolumn{1}{c}{\multirow{2}[0]{*}{\tablehead{Backbone}}} & \multicolumn{2}{c}{\tablehead{Conditioning ablation}} & \multicolumn{4}{c}{\tablehead{Close-target $\lambda$ sweep}} \\
    \cmidrule(lr){2-3}\cmidrule(l){4-7}
    & \tablehead{Close $\Delta\IC\uparrow$} & \tablehead{Volume $\Delta\IC\uparrow$} & \tablehead{$\IC(0)\uparrow$} & \tablehead{$\IC(1)\uparrow$} & \tablehead{$\MSE(0)\downarrow$} & \tablehead{$\MSE(1)\downarrow$} \\
    \midrule
    DLinear & $-0.001\mathbin{\pm}0.003$ & $+0.003\mathbin{\pm}0.004$ & $0.072\mathbin{\pm}0.016$ & $0.127\mathbin{\pm}0.004$ & $(4.74\mathbin{\pm}1.02)\!\times10^{-4}$ & $(5.20\mathbin{\pm}1.84)\!\times10^{-4}$ \\
    PatchTST & $+0.006\mathbin{\pm}0.002$ & $+0.000\mathbin{\pm}0.001$ & $0.049\mathbin{\pm}0.045$ & $0.122\mathbin{\pm}0.008$ & $(4.18\mathbin{\pm}0.03)\!\times10^{-4}$ & $(5.02\mathbin{\pm}0.32)\!\times10^{-4}$ \\
    iTransformer & $+0.002\mathbin{\pm}0.001$ & $+0.001\mathbin{\pm}0.001$ & $0.092\mathbin{\pm}0.013$ & $0.127\mathbin{\pm}0.004$ & $(4.15\mathbin{\pm}0.03)\!\times10^{-4}$ & $(5.53\mathbin{\pm}0.60)\!\times10^{-4}$ \\
    TimeMixer & $-0.001\mathbin{\pm}0.002$ & $+0.002\mathbin{\pm}0.000$ & $0.114\mathbin{\pm}0.005$ & $0.127\mathbin{\pm}0.003$ & $(4.11\mathbin{\pm}0.01)\!\times10^{-4}$ & $(5.24\mathbin{\pm}0.81)\!\times10^{-4}$ \\
    Autoformer & $+0.006\mathbin{\pm}0.004$ & $-0.004\mathbin{\pm}0.011$ & $0.041\mathbin{\pm}0.034$ & $0.105\mathbin{\pm}0.003$ & $(4.17\mathbin{\pm}0.02)\!\times10^{-4}$ & $(7.86\mathbin{\pm}2.31)\!\times10^{-4}$ \\
    Transformer & $-0.002\mathbin{\pm}0.002$ & $+0.006\mathbin{\pm}0.004$ & $0.006\mathbin{\pm}0.013$ & $0.120\mathbin{\pm}0.002$ & $(4.37\mathbin{\pm}0.24)\!\times10^{-4}$ & $(5.91\mathbin{\pm}0.50)\!\times10^{-4}$ \\
    Informer & $-0.023\mathbin{\pm}0.047$ & $+0.001\mathbin{\pm}0.003$ & $-0.003\mathbin{\pm}0.005$ & $0.105\mathbin{\pm}0.005$ & $(4.36\mathbin{\pm}0.21)\!\times10^{-4}$ & $(8.26\mathbin{\pm}1.17)\!\times10^{-4}$ \\
    FEDformer & $-0.001\mathbin{\pm}0.004$ & $+0.001\mathbin{\pm}0.003$ & $0.051\mathbin{\pm}0.048$ & $0.117\mathbin{\pm}0.004$ & $(4.17\mathbin{\pm}0.02)\!\times10^{-4}$ & $(5.77\mathbin{\pm}1.89)\!\times10^{-4}$ \\
    Nonstat.\ Transformer & $-0.000\mathbin{\pm}0.000$ & $+0.002\mathbin{\pm}0.002$ & $0.059\mathbin{\pm}0.061$ & $0.147\mathbin{\pm}0.002$ & $(4.15\mathbin{\pm}0.04)\!\times10^{-4}$ & $(4.26\mathbin{\pm}0.28)\!\times10^{-4}$ \\
    Reformer & $-0.002\mathbin{\pm}0.001$ & $+0.004\mathbin{\pm}0.002$ & $0.012\mathbin{\pm}0.032$ & $0.121\mathbin{\pm}0.002$ & $(4.62\mathbin{\pm}0.43)\!\times10^{-4}$ & $(5.82\mathbin{\pm}1.17)\!\times10^{-4}$ \\
    LightTS & $+0.000\mathbin{\pm}0.000$ & $-0.001\mathbin{\pm}0.006$ & $0.049\mathbin{\pm}0.011$ & $0.128\mathbin{\pm}0.001$ & $(4.18\mathbin{\pm}0.02)\!\times10^{-4}$ & $(5.03\mathbin{\pm}1.04)\!\times10^{-4}$ \\
    TSMixer & $-0.000\mathbin{\pm}0.001$ & $+0.001\mathbin{\pm}0.001$ & $0.060\mathbin{\pm}0.031$ & $0.130\mathbin{\pm}0.001$ & $(4.16\mathbin{\pm}0.01)\!\times10^{-4}$ & $(5.35\mathbin{\pm}1.43)\!\times10^{-4}$ \\
    \bottomrule
  \end{tabular}%
  }
\end{table}

\subsection{\texorpdfstring{$\lambda$}{lambda}-sweep by model}

\paragraph{A similar pattern across model families.}
The endpoint columns of \Cref{tab:arch-ablation-sweep}, together with the
seven-point sweeps, show the same pattern across backbones: the first nonzero ranking weight
produces almost the entire correlation gain, after which correlation plateaus
while squared error and amplitude keep rising. Differences in curvature reflect
backbone stability, and Informer has visibly larger seed dispersion. The table
keeps the endpoints readable, while Figure~1(b) displays the intermediate
operating points.

Moving from $\lambda=0$ to $\lambda=1$ raises close IC for every backbone; the
median gain is $0.0489$ at a median $1.280\times$ MSE multiplier. Most models
obtain nearly all of that ranking gain by $\lambda=0.25$--$0.5$, while larger
weights mainly increase MSE. This saturation occurs because once
cross-sectional correlation saturates, stronger correlation pressure adds
little and continues relaxing scale discipline. Non-stationary
Transformer is the clearest favorable case: its ranking gain is obtained with
almost no MSE penalty, suggesting that its normalization already supplies part
of the calibration control.

\FloatBarrier
\section{Discussion}
\label{sec:discussion}

\paragraph{Account for target predictability in evaluation.}
The same correlation can be weak on a predictable target and strong on an
unpredictable one. A predictability ceiling makes that distinction explicit,
provided it is estimated under a declared information set and at the same
aggregation level as the reported score. In a panel, this means comparing
per-timestamp cross-sectional IC with per-timestamp ceilings before averaging. We use
the pooled identity only to characterize pooled amplitude attenuation and do
not convert it into a percentage of the time-averaged cross-sectional IC.

\paragraph{Evaluate amplitude and joint structure.}
GIFT-Eval scores one series at a time. Although 43 of its 97 configurations are
multivariate, each is flattened before scoring, and none of its eleven metrics
measures raw amplitude or cross-sectional structure. A complete evaluation
should report $\ampraw$, a cross-sectional dependence score, and ordinary
per-series error. The first identifies attenuation or overshoot. The second asks
whether the forecast preserves relationships used by a downstream decision.
Neither can be reconstructed from a scalar average of per-series errors.

A multivariate benchmark can add these measurements without replacing its
existing metrics. It must preserve channel grouping and timestamps through
evaluation, compute the ordering score within each cross-section, and average
only afterward. It should also state whether amplitude is measured on raw output or after a fitted calibration map. These
choices prevent pooled correlation from mixing temporal and cross-sectional
effects and prevent post-hoc rescaling from hiding what the model outputs.

\paragraph{Evaluate the missing cross-series functional.}
Finance is a clear testbed because portfolio selection uses the ordering of a
cross-section, but the issue is broader. In environmental sensor networks,
clinical monitoring, energy systems, and traffic forecasting, decisions depend
on which channels move together or which deteriorate first. Benchmarks for
these settings should retain their multivariate structure through scoring.
Depending on the application, the relevant relationship may be correlation,
concordance, tail dependence, or a learned graph.

The target determines which metric belongs in the protocol. A portfolio selector
uses rank or correlation; an energy dispatcher may require calibrated joint
tails; a clinical alert system may care about which channels cross thresholds
together. There is no universal dependence score. The requirement is that
evaluation retain the relationship the decision uses instead of flattening the
problem into independent series.

\paragraph{Where to expect forecast collapse.}
The analysis gives two diagnostics rather than a universal claim about TSFMs.
First, estimate how much variance a reasonable set of baselines can recover under the
same information set. Low achieved $R^{2}$ is a warning sign for amplitude
attenuation, although it cannot prove that the true ceiling is low. Second,
identify whether the downstream task needs a relationship across channels
that the training and benchmark losses score. The strongest failure should
appear when both conditions hold: little predictable variance and an important
multivariate functional absent from the objective. Short-horizon returns are
one instance; demand residuals after strong seasonal adjustment, local sensor
innovations, and differences between closely related measurements are other
plausible targets.

The same diagnostics make the claim falsifiable. If a target clearly has
high predictable variance and a well-fitted MSE model still produces a nearly flat
forecast, the predictability explanation is insufficient and optimization,
regularization, or preprocessing becomes the leading suspect. If a pointwise
model preserves the relevant cross-sectional structure across architectures
and seeds, then the coupling invariance has not produced an empirical
failure in that setting. Conversely, a cross-sectional loss should improve the
targeted functional under a matched protocol, as it does for every backbone in
\Cref{tab:arch}. These checks separate a property of the statistical problem
from a defect in a particular implementation.

\paragraph{Choose the operating point on validation data.}
Increasing $\lambda$ initially recovers most of the available ordering and then
primarily increases amplitude and squared error. This creates a practical
selection rule: choose $\lambda$ on a validation split taken by date to meet a
declared amplitude tolerance, then report the untouched test result. The present
paper reports the prespecified $\lambda=1$ setting and treats the seven-point
sweep as a measured frontier, not as test-set model selection. The chosen operating point is
application-dependent: a ranking system may tolerate amplitude
error, whereas a portfolio or physical forecast may not.

\paragraph{Implications for model selection.}
Per-series error remains necessary: a model with good ordering can still produce
unusable magnitudes, as the IC-only point in \Cref{fig:section8-evidence}(a) shows. The reverse is
also true: an apparently calibrated marginal forecast may carry little
cross-sectional discrimination. Reporting both axes changes model selection
from a single leaderboard rank into a constrained choice. A practitioner can
set an admissible calibration range and maximize structure within it, or set a
minimum ordering score and minimize error. The frontier communicates the
available tradeoff without asserting that one scalar weight is correct for
every application.

\FloatBarrier
\section{Conclusion and Limitations}
\label{sec:conclusion}

Forecast collapse combines an amplitude failure with an ordering failure.
Predictability limits the amplitude of a best-scaled point forecast. Separately,
per-series risk does not identify cross-series coupling, even though
cross-sectional decisions depend on that coupling. Finance1K shows both
patterns: MSE forecasts are nearly flat, IC-only forecasts are badly
miscalibrated, and \method improves cross-sectional correlation while retaining
a calibration anchor.
TimesFM and Chronos show a parallel association between raw amplitude and
achieved predictability across 97 public benchmark configurations. Evaluation
should preserve and score multivariate structure as well as per-series error.

\paragraph{Practical consequence.}
A small forecast is not automatically a weak model: it can be the correctly
calibrated response to a target with little predictable variance. Nor does a
low per-series error guarantee that the forecast preserves the relationships a
decision needs. The useful audit is therefore two-dimensional. Measure what
the model outputs before calibration, and measure the relevant structure at the
level where the decision is made. If amplitude is attenuated in line with
predictability, changing architecture alone should not remove it. If structure
is missing, adding a matched cross-sectional term should matter more than rescaling
the output. This separation turns "flat forecasts" from a visual symptom into
two testable questions about information and objective design.

\paragraph{Limitations.}
The two parts of the argument have different empirical reach. The amplitude
identity is exact for best-scaled forecasts and is checked in simulation; the
raw outputs of two pretrained models show the predicted association across a
large public benchmark. Achieved $R^{2}$ from a fitted set of baselines is a lower
bound on available predictability, however, so the GIFT-Eval experiment does
not estimate the true ceiling or test the inequality in
\Cref{prop:collapse}. It shows where amplitude collapse appears in practice.

The cross-sectional diagnosis and remedy are demonstrated on one real panel.
The twelve-backbone experiment reduces the chance that the result is tied to
one architecture, but does not establish the same ordering failure in another
domain. Finance1K is stratified by volatility and requires long coverage. It is useful for testing predictability
and cross-sectional ranking, but it does not support a market-wide performance
claim. The Finance1K release package provides the aligned panels, exact
chronological split, ticker order, validity masks, and checksums needed to
reproduce the data contract.

\method targets one dependence functional, per-timestamp Pearson correlation. It
does not estimate a joint distribution, and applications concerned with rank,
tails, or calibrated scenarios may require a different cross-sectional term. The
frontier is measured for a simple scalarization and is not fundamental;
regression-compatible listwise objectives may dominate it. The remedy also
requires a meaningful cross-section at each timestamp and does not directly apply
to a single series.

Finally, aggregation matters under drift. Squared error pools by linearity,
whereas IC is a ratio and can show Simpson effects when pooled over time
\citep{robinson1950ecological}. A ceiling computed from average predictability
is also Jensen-loose because the square-root map in \eqref{eq:volpred} is
concave. Our empirical IC is therefore computed at each timestamp and averaged,
and we avoid comparing it with a pooled ceiling percentage. Time-varying
predictability changes where attenuation is strongest, while the per-timestamp identity in Proposition~\ref{prop:collapse} remains unchanged.

\section{Generative AI Usage}
\label{sec:genai}

We disclose all uses of generative AI in the preparation of this work. The
authors remain responsible for all of its content.

Generative AI assistants were used in two ways: to write and refactor code,
including the experiment harness, the figure and table renderers, and the
validation suite; and for writing assistance, including drafting, editing, and
improving clarity. They were not used to generate research ideas, to develop
the theory or its proofs, to design the experiments, or to interpret results.

All AI-assisted text and code was reviewed and edited by the authors, who
accept responsibility for its accuracy and integrity.

\bibliographystyle{assets/plainnat}
\bibliography{references}

\end{document}